\documentclass[conference]{IEEEtran}
\usepackage{multirow}
\usepackage[table]{xcolor}
\usepackage[numbers]{natbib}
\usepackage{multicol}
\usepackage[bookmarks=true]{hyperref}
\usepackage{amsmath,amsfonts,amsthm,amssymb}
\usepackage{dsfont}
\allowdisplaybreaks
\usepackage[caption=false,font=footnotesize,labelfont=sf,textfont=sf]{subfig}
\usepackage[T1]{fontenc}
\usepackage{graphicx}

\usepackage{nicematrix}
\usepackage{float} 
\graphicspath{{Images/}} 
\usepackage{calc} 
\let\labelindent\relax
\usepackage{enumitem} 
\usepackage{xspace}
\usepackage{cleveref}
\usepackage{bm}
\usepackage{tabularx}
\usepackage{multirow}
\usepackage[noend,ruled,linesnumbered]{algorithm2e}\RestyleAlgo{ruled}
 \SetCommentSty{mycommfont}
 \SetKw{Return}{return}
 \SetKw{Continue}{continue}

\begin{document}
\title{Effective Sampling for Robot Motion Planning 
Through the Lens of Lattices}


\author{\authorblockN{Itai Panasoff and Kiril Solovey}
\authorblockA{Viterbi Faculty of Electrical and Computer Engineering\\
Technion--Israel Institute of Technology, Haifa, Israel\\
itaip@campus.technion.ac.il, kirilsol@technion.ac.il}
}

\maketitle

\newcommand{\ignore}[1]{}

\def\P{\mathcal{P}} \def\C{\mathcal{C}} \def\H{\mathcal{H}}
\def\F{\mathcal{F}} \def\U{\mathcal{U}} \def\L{\mathcal{L}}
\def\O{\mathcal{O}} \def\I{\mathcal{I}} \def\S{\mathcal{S}}
\def\G{\mathcal{G}} \def\Q{\mathcal{Q}} \def\I{\mathcal{I}}
\def\T{\mathcal{T}} \def\L{\mathcal{L}} \def\N{\mathcal{N}}
\def\V{\mathcal{V}} \def\B{\mathcal{B}} \def\D{\mathcal{D}}
\def\W{\mathcal{W}} \def\R{\mathcal{R}} \def\M{\mathcal{M}}
\def\X{\mathcal{X}} \def\A{\mathcal{A}} \def\Y{\mathcal{Y}}
\def\L{\mathcal{L}}

\def\dS{\mathbb{S}} \def\dT{\mathbb{T}} \def\dC{\mathbb{C}}
\def\dG{\mathbb{G}} \def\dD{\mathbb{D}} \def\dV{\mathbb{V}}
\def\dH{\mathbb{H}} \def\dN{\mathbb{N}} \def\dE{\mathbb{E}}
\def\dR{\mathbb{R}} \def\dM{\mathbb{M}} \def\dm{\mathbb{m}}
\def\dB{\mathbb{B}} \def\dI{\mathbb{I}} \def\dM{\mathbb{M}}
\def\dZ{\mathbb{Z}}

\def\open{\textup{OPEN}}
\def\opennext{\textup{OPEN}_\textup{next}}
\def\visited{\textup{VISITED}}

\def\E{\mathbf{E}} 

\def\eps{\varepsilon}
\def\epsilon{\varepsilon}

\def\itai#1{\textcolor{blue}{(\textbf{Itai:} #1})}
\def\kiril#1{\textcolor{red}{(\textbf{Kiril:} #1})}
\def\yaniv#1{\textcolor{cyan}{(\textbf{Yaniv:} #1})}
\def\ido#1{\textcolor{brown}{(\textbf{Ido:} #1})}
\def\roy#1{\textcolor{teal}{(\textbf{Roy:} #1})}

\def\niceparagraph#1{\vspace{5pt} \noindent \textbf{#1}}

\def\dt{\,\mathrm{d}t}
\def\dx{\,\mathrm{d}x}
\def\dy{\,\mathrm{d}y}
\def\drho{\,\mathrm{d}\rho}

\theoremstyle{definition}
\newtheorem{definition}{Definition}
\newtheorem{problem}{Problem}
\theoremstyle{theorem}
\newtheorem{lemma}{Lemma}
\newtheorem{cor}{Corollary}
\newtheorem{thm}{Theorem}
\newtheorem{claim}{Claim}

\newcommand{\prm}{{\tt PRM}\xspace}
\newcommand{\prmstar}{{\tt PRM}$^*$\xspace}
\newcommand{\rrt}{{\tt RRT}\xspace}
\newcommand{\est}{{\tt EST}\xspace}
\newcommand{\grrt}{{\tt GEOM-RRT}\xspace}
\newcommand{\rrtstar}{{\tt RRT}$^*$\xspace}
\newcommand{\rrg}{{\tt RRG}\xspace}
\newcommand{\btt}{{\tt BTT}\xspace}
\newcommand{\fmt}{{\tt FMT}$^*$\xspace}
\newcommand{\dfmt}{{\tt DFMT}$^*$\xspace}
\newcommand{\dprm}{{\tt DPRM}$^*$\xspace}
\newcommand{\mstar}{{\tt M}$^*$\xspace}
\newcommand{\drrtstar}{{\tt dRRT}$^*$\xspace}
\newcommand{\sst}{{\tt SST}\xspace}
\newcommand{\sststar}{{\tt SST}$^*$\xspace}
\newcommand{\stride}{{\tt STRIDE}\xspace}
\newcommand{\aorrt}{{\tt AO-RRT}\xspace}
\newcommand{\aorrtrebuilding}{{\tt Multi-tree AO-RRT}\xspace}
\newcommand{\aorrtnopruning}{{\tt AO-RRT}\xspace}
\newcommand{\aoest}{{\tt AO-EST}\xspace}
\newcommand{\kpiece}{{\tt KPIECE}\xspace}
\newcommand{\hybrrttwo}{{\tt HybAO-RRT}\xspace}
\newcommand{\hybrrttwostride}{{\tt HybAO-RRT-STRIDE}\xspace}
\newcommand{\hybrrttwoest}{{\tt HybRRT2\!.\!0-EST}\xspace}
\newcommand{\rrttwo}{{\tt AO-RRT}\xspace}
\newcommand{\aorrtprune}{{\tt AO-RRT Pruning}\xspace}
\newcommand{\rrtbc}{{\tt BCRRT}\xspace}
\newcommand{\rrtbctwo}{{\tt BCRRT2\!.\!0}\xspace}

\newcommand{\xmin}{x_{\textup{min}}}
\newcommand{\Xnear}{X_{\textup{near}}}
\newcommand{\Xgoal}{\X_{\textup{goal}}}
\newcommand{\xgoal}{x_{\textup{goal}}}
\newcommand{\xinit}{x_{\textup{init}}}
\newcommand{\xnew}{x_{\textup{new}}}
\newcommand{\xnear}{x_{\textup{near}}}
\newcommand{\xrand}{x_{\textup{rand}}}
\newcommand{\xrandtwo}{x_{\textup{rand2}}}
\newcommand{\ygoal}{y_{\textup{goal}}}
\newcommand{\Ygoal}{\Y_{\textup{goal}}}
\newcommand{\Yfree}{\Y_{\textup{free}}}
\newcommand{\Xfree}{\X_{\textup{free}}}
\newcommand{\yinit}{y_{\textup{init}}}
\newcommand{\ynew}{y_{\textup{new}}}
\newcommand{\ynear}{y_{\textup{near}}}
\newcommand{\yrand}{y_{\textup{rand}}}
\newcommand{\ymin}{y_{\textup{min}}}
\newcommand{\xparent}{x_{\textup{parent}}}
\newcommand{\cmin}{c_{\textup{min}}}
\newcommand{\cmax}{c_{\textup{max}}}
\newcommand{\crand}{c_{\textup{rand}}}
\newcommand{\cnew}{c_{\textup{new}}}
\newcommand{\cnear}{c_{\textup{near}}}
\newcommand{\Tprop}{T_{\textup{prop}}}
\newcommand{\trand}{t_{\textup{rand}}}
\newcommand{\tnew}{t_{\textup{new}}}
\newcommand{\urand}{u_{\textup{rand}}}
\newcommand{\unew}{u_{\textup{new}}}
\newcommand{\pinew}{\pi_{\textup{new}}}
\newcommand{\pimin}{\pi_{\textup{min}}}

\newcommand{\zv}{\vec{0}}

\newcommand{\randomstate}{\textsc{random-state}}
\newcommand{\sample}{\textsc{sample}}
\newcommand{\nearest}{\textsc{nearest}}
\newcommand{\near}{\textsc{near}}
\newcommand{\steer}{\textsc{steer}}
\newcommand{\collisionfree}{\textsc{collision-free}}
\newcommand{\propagate}{\textsc{propagate}}
\newcommand{\newstate}{\textsc{new-state}}
\newcommand{\propstate}{\textsc{prop-state}}
\newcommand{\propcost}{\textsc{prop-cost}}
\newcommand{\cost}{\textsc{cost}\xspace}
\newcommand{\tracepath}{\textsc{trace-path}}
\newcommand{\nulll}{\textsc{null}}

\newcommand{\de}{\delta,\epsilon}
\newcommand{\decomp}{$(\delta,\epsilon)$-complete\xspace}
\newcommand{\decomps}{$(\delta,\epsilon)$-completeness\xspace}

\newcommand{\ZN}{\mathbb{Z}^d}
\newcommand{\DN}{D_d^*}
\newcommand{\AN}{A_d^*}
\newcommand{\Lattices}{$\ZN,\DN,\AN$ }
\newcommand{\XL}{\X_{\Lambda}^{\delta,\epsilon}}
\newcommand{\XZ}{\X_{\mathbb{Z}^d}^{\delta,\epsilon}}
\newcommand{\XD}{\X_{D_d^*}^{\delta,\epsilon}}
\newcommand{\XA}{\X_{A_d^*}^{\delta,\epsilon}}
\newcommand{\XR}{\X_{\text{rnd}}^{\delta,\epsilon}}

\newcommand{\td}{\tilde{d}}

\newcommand{\vol}{\textup{vol}}

\newcommand{\btheta}{\Bar{\theta}}

\newcommand{\Rcov}{\R_{\text{cover}}}

\newcommand{\glo}{\textsc{glo}\xspace}
\newcommand{\loc}{\textsc{loc}\xspace}
\newcommand{\rnd}{\textsc{rnd}\xspace}
\newcommand{\rndm}{\textsc{rnd}$^-$\xspace}

\newif\ifincludeappendix
\includeappendixtrue  
\newcommand{\appendixtext}[1]{%
  \ifincludeappendix
    in~\ref{#1}.%
  \else
    in the supplementary material)%
  \fi
}
\newcommand{\conditionaltext}[2]{%
  \ifincludeappendix
    #1%
  \else
    #2%
  \fi
}


\begin{abstract}
Sampling-based methods for motion planning, which capture the structure of the robot's free space via (typically random) sampling, have gained popularity due to their scalability, simplicity, and for offering global guarantees, such as probabilistic completeness and asymptotic optimality. Unfortunately, the practicality of those guarantees remains limited as they do not provide insights into the behavior of motion planners for a finite number of samples (i.e., a finite running time). In this work, we harness lattice theory and the concept of $\bm{(\delta,\eps)}$-completeness by Tsao et al.~(2020) to construct deterministic sample sets that endow their planners with strong finite-time guarantees while minimizing running time. In particular, we introduce a highly-efficient deterministic sampling approach based on the $\bm{A_d^*}$ lattice, which 
is the best-known geometric covering in dimensions $\bm{\leq 21}$. Using our new sampling approach, we obtain at least an order-of-magnitude speedup over existing deterministic and uniform random sampling methods for complex motion-planning problems. Overall, our work provides deep mathematical insights while advancing the practical applicability of sampling-based motion planning. 
\href{https://github.com/MRSTechnion/lattice-sampling-mp}{https://github.com/MRSTechnion/lattice-sampling-mp}
\end{abstract} 

\IEEEpeerreviewmaketitle

\section{Introduction}
Motion planning is a key ingredient in autonomous robotic systems, whose aim is computing collision-free trajectories for a robot operating in environments cluttered with obstacles~\cite{lavalle2006planning}. 
Over the years, various approaches have been developed for tackling the problem, including potential fields~\cite{luo2024potential}, geometric methods~\cite{halperin2017algorithmic}, and optimization-based approaches~\cite{SchulmanDHLABPPGA14,MalyutaEtAl2022,MarcucciEA23}. 
In this work, we focus on sampling-based planners (SBPs), which aim to capture the structure of the robot's free space through graph approximations that result from configuration sampling (typically in a random fashion) and connecting nearby samples. 
SBPs have enjoyed popularity in recent years due to their relative scalability, in terms of the number of robot degrees of freedom (DoFs), and the ease of their implementation~\cite{OrtheyCK24}. 

\begin{figure*}[h!]
  \centering
  \subfloat[$\X_{\dZ_2}^{\delta,\epsilon}$ sample set.]{
    \includegraphics[width=0.27\textwidth, trim={2.2cm 1.7cm 0.9cm 1.0cm},clip]{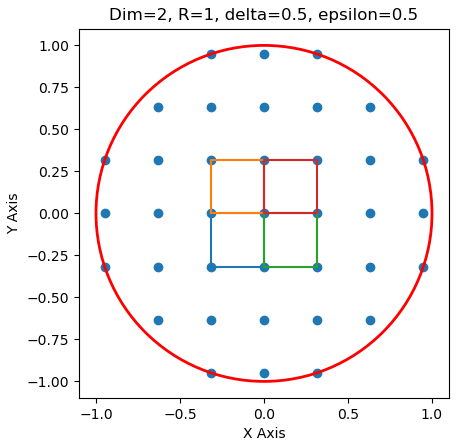}
    }
  \hfil
  \subfloat[$\X_{D_2^*}^{\delta,\epsilon}$ sample set.]{
    \includegraphics[width=0.27\textwidth, trim={2.2cm 1.8cm 0.9cm 1.0cm},clip]{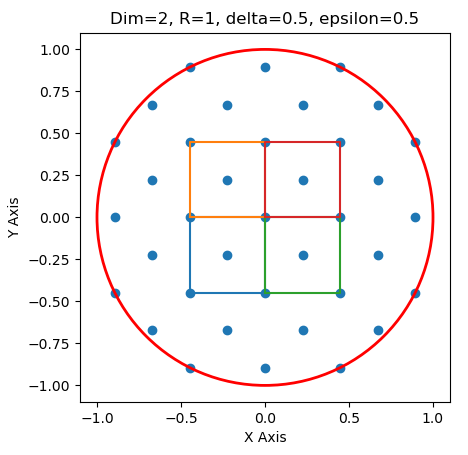}
    }
  \hfil
  \subfloat[$\X_{A_2^*}^{\delta,\epsilon}$  sample set.]{
    \includegraphics[width=0.27\textwidth, trim={2.4cm 1.7cm 0.9cm 1.0cm},clip]{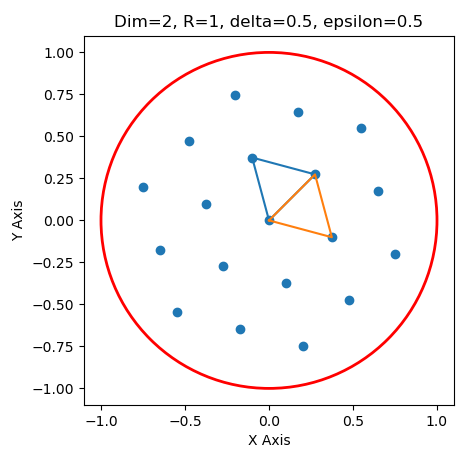}
    }
  \caption{Sample sets within a fixed disc in $\dR^2$, derived from the lattices $\dZ^2, D_2^*$ and $A^*_2$, which yield \decomp guarantees for the same values of $\delta$ and $\eps$. The set $\X_{\dZ_2}^{\delta,\epsilon}$ can be viewed as a tessellation of space using cubes. The set $\X_{D_2^*}^{\delta,\epsilon}$ is obtained by placing a (rescaled) standard grid, and then placing another point in the middle of each cube. The set $\X_{A_2^*}^{\delta,\epsilon}$ can be viewed as a rescaled hexagonal grid as each point is surrounded by a hexagon whose vertices are points in the set. Note that the density of $\X_{\dZ^2}^{\delta,\eps}$ and $\X_{D^*_2}^{\delta,\eps}$ is the same, and higher than the density of $\X_{A^*_2}^{\delta,\eps}$.}
  \label{fig:2d_lattices}
\end{figure*}

Another key benefit is the ability of SBPs to escape local minima (unlike potential fields) and global solution guarantees (in contrast, optimization-based approaches~\cite{SchulmanDHLABPPGA14}, which typically provide only local guarantees). Earlier work on the theoretical foundations of SBPs has focused on deriving probabilistic completeness (PC) guarantees for methods such as PRM~\cite{kavraki1996probabilistic} or RRT~\cite{LaVKuf01,KunzS14,Kleinbort.Solovey.ea.19}. PC implies that the probability of a given planner finding a solution (if one exists) converges to one as the number of samples tends to infinity. The work of~\citet{karaman2011sampling} initiated studying the quality of the solution returned by SBPs. Specifically, they introduced the planners PRM* and RRT*, and proved that the solution length of those planners converges to the optimum as the number of samples tends to infinity---a property called asymptotic optimality (AO). Subsequent work has introduced even more powerful AO planners for geometric~\cite{JSCP15,GammellBS20} and dynamical~\cite{HauserZ16,LiETAL16} systems.

Unfortunately, the practical relevance of the aforementioned theoretical findings remains limited due to the lack of meaningful finite-time implications. Specifically, when a solution is obtained using a finite number of samples, it is unclear to what extent its quality can be improved with additional computation time. Moreover, in cases where no solution is returned, it is uncertain whether a solution does not exist or if the algorithm simply failed to find one. Developing finite-time bounds through randomized sampling continues to be a significant challenge~\cite{DobsonMB15,shaw2024towards}.

Deterministic sampling methods such as grid sampling or Halton sequences~\cite{lavalle2006planning}, where samples are generated according to a geometric principle, can improve the performance of SBPs in practice and simplify the algorithm analysis. Specifically, some deterministic sampling procedures have a significantly lower dispersion than uniform random sampling, which implies that the former requires fewer samples to cover the search space to a desired resolution~\cite{janson2018deterministic}. 
Recently, Tsao et al.~\cite{tsao2020sample} have leveraged deterministic sampling to disrupt the asymptotic analysis paradigm by introducing a significantly stronger notion than AO, called \decomps, that yields finite-time guarantees for PRM-based algorithms such as PRM*~\cite{karaman2011sampling}, FMT*~\cite{JSCP15}, BIT*~\cite{GammellBS20}, and GLS~\cite{MandalikaCSS19}. Informally, a \emph{finite} sample set is \decomp for a given approximation factor $\eps>0$ and clearance parameter $\delta>0$, if the corresponding planner returns a solution whose length is at most $(1+\eps)$ times the length of the shortest $\delta$-clear solution. If no solution is found using a \decomp sample set then no solution of clearance $\delta$ exists. 

The work of~\citet{tsao2020sample} derived a relation between \decomps and geometric space coverage to obtain lower bounds on the number of samples necessary to achieve \decomps, as well as upper bounds accompanied with explicit (deterministic) sampling distributions. A follow-up work by~\citet{dayan2023near} has introduced an even more compact \decomp sample distribution that is more efficient than the one proposed in~\cite{tsao2020sample} or rectangular grid sampling. In particular, the staggered grid~\cite{dayan2023near} consists of two shifted and rescaled copies of the rectangular grid (see Figure~\ref{fig:2d_lattices} and Figure~\ref{fig:3d_lattices}). 

However, the work~\cite{dayan2023near} still leaves a significant gap between the lower bound in~\cite{tsao2020sample} and the upper bound obtained with the staggered grid. In practice, this gap limits the applicability of the \decomps theory to relatively low dimensions (up to dimension 6) due to the large number of samples currently needed to satisfy this property, which can lead to excessive running times. 

\vspace{5pt}
\noindent \textbf{Contribution.} In this work, we develop a theoretical framework for obtaining highly-efficient \decomp sample sets by leveraging the foundational theory of lattices\footnote{Lattices are point sets exhibiting a regular geometric structure, which are obtained by transforming the integer lattice $\dZ^d$. For instance, the aforementioned rectangular grid and the staggered grid can be viewed as lattices.}~\cite{conway2013sphere}, which has been instrumental in diverse areas from number theory~\cite{siegel_geometry_numbers}, coding theory~\cite{ebeling2013lattices}, and crystallography~\cite{sands1994introduction}. Specifically, we show that lattices can be transformed to obtain \decomp sample sets (Theorem~\ref{thm:decomp_lattices}) and develop tight theoretical bounds on their size (Theorem~\ref{thm:general_sample_complexity}), which allows to compare between different sample sets qualitatively. 
Using this machinery, we not only refine and generalize previous results on the staggered grid~\cite{dayan2023near} but also introduce a new highly efficient \decomp sample set that is based on the $\AN$ lattice, which is famous for its minimalist coverage properties~\cite{conway2013sphere}. We also initiate the study of a new property, which estimates the computational cost resulting from using a given sample set in a more informative manner than sample complexity. In particular, the property called collision-check complexity captures the amount of collision checks, which is typically a computational bottleneck.

From a practical perspective, when solving motion-planning problems using lattice-based sample sets, we show that our $\AN$-based sample sets can result in at least order-of-magnitude improvement in terms of running time over staggered-grid samples and two orders of magnitude improvements over rectangular grids. Moreover, $\AN$-based sample sets are vastly superior in practice to the widely-used uniform random sampling, which is evident in improved running times, success rates, and solution quality.

\vspace{5pt}
\noindent \textbf{Organization.} In Section~\ref{sec:preliminaries} we review basic definitions on motion planning and \decomps, and formally define our objectives. In Section~\ref{sec:lattices}, we develop a general tool for transforming lattices into \decomp sample sets. We obtain sample-complexity bounds for lattice-based sample sets in Section~\ref{sec:sample_complexity}, and generalize those bounds to collision-check complexity in Section~\ref{sec:collision_complexity}. We evaluate the practical implications of our theory in Section~\ref{sec:experiments}, and conclude with a discussion of limitations and future directions in Section~\ref{sec:future}.

\section{Preliminaries}\label{sec:preliminaries}
The motion-planning problem concerns computing a collision-free path for a robot in an environment cluttered with obstacles. We concider a holonomic robot with a configuration space $\C = \dR^d$. 
The dimension $d\geq 2$ represents the DoF and is finite. A motion planning problem is a tuple $\M:= (\C_f, q_{s},  q_{g})$, where $\C_f\subseteq \C$ is the free space (the set of collision-free configurations), and $q_s,q_g\in \C_f$ are the start and goal configurations, respectively. A solution for $\M$ is a continuous collision-free path $\pi:[0,1]\to\C_f$ that begins at $\pi(0) = q_{s}$ and ends at $\pi(1) = q_g$.

Two critical properties of a given path $\pi$ for a problem $\M= (\C_f, q_{s},  q_{g})$, are its length $\ell(\pi)\geq 0$, and its clearance. For a given value $\delta\geq 0$, we say that the path $\pi$ is $\delta$-clear if $\bigcup_{0 \leq t \leq 1}\B_{\delta}(\pi(t)) \subseteq \C_f$, where $\B_\rho(p)$ is the $d$-dimensional closed Euclidean ball with radius $\rho>0$ centered at $p\in\mathbb{R}^d$. We denote $\B_\rho:=\B_\rho(o)$, where $o$ is the origin of $\dR^d$.

\subsection{Probabilistic roadmaps and completeness}
We present a formal definition of the Probabilistic Roadmap (PRM) method \cite{kavraki1996probabilistic}, which constructs a discrete graph capturing the connectivity of $\C_f$ through sampling. Albeit sampling usually refers to a randomized process, here we consider deterministic sampling, as was recently done in~\cite{tsao2020sample,dayan2023near}.\footnote{With deterministic sampling, the term ``probabilistic'' in PRM can seem misleading. Nevertheless, we choose to stick to PRM considering its popularity and the underlying graph structure it represents, which is a key to our analysis.} We emphasize that our analysis below is not confined to PRMs, and applies to various PRM-based planners, as mentioned above. 

For a given motion planning problem $\M = (\C_f,q_s,q_g)$, a sample (point) set $\X\subset \C$, and a connection radius $r>0$, PRM generates a graph denoted by $G_{\mathcal{M}(\X, r)} = (V,E)$. The vertex set $V$ consists of all collision-free configurations in $\X \cup \{ q_s, q_g\}$. The set of undirected edges, $E$, consists of all the vertex pairs $u, v\in V$ such that the Euclidean distance between them is at most $r$, and the straight-line segment $\overline{uv}$ between them is collision-free. That is, 
\begin{align}
V := & (\X\cup \{ q_s, q_g\}) \cap \C_f, \nonumber \\
E := & \left\{ \{v, u\} \in V\times V : \lVert v-u\rVert\leq r, \overline{uv} \subset \C_f \right\}. \label{eq:graph}
\end{align}

In this work, we are interested in obtaining sample sets and connection radii for PRM that achieve a desired solution quality in terms of path length. Unlike most theoretical results for SBP, which consider asymptotic guarantees, here we rely on a stronger deterministic notion.

\begin{definition}[($\delta,\epsilon$)-completeness~\cite{tsao2020sample}] Given a sample set $\X\subset \C$ and connection radius $r>0$, the pair $(\X, r)$ is ($\delta,\epsilon$)-complete for a clearance parameter $\delta>0$ and stretch factor $\epsilon>0$, if for \emph{every} $\delta$-clear problem $\mathcal{M}=(\C_f, q_s, q_g)$, the  graph $G_{\mathcal{M}(\X, r)}$ contains a path from $q_s$ to $q_g$ with length at most $(1+\eps)$ times the optimal $\delta$-clear length, denoted by $\text{OPT}_\delta$. That is, it holds that
\[
    \ell(G_{\mathcal{M}(\X, r)},q_s, q_g)\leq (1+\epsilon)\text{OPT}_\delta,
\]
where $\ell(G_{\mathcal{M}(\X, r)},q_s, q_g)$ denotes the length of the shortest path from $q_s$ to $q_g$ in the graph $G_{\mathcal{M}(\X, r)}$.
\end{definition}

The property of ($\delta,\epsilon$)-completeness has several key advantages over asymptotic notions, such as PC and AO. First, there exists (with probability $1$) a \emph{finite} sample set $\X$ and radius $r\in (0,\infty)$  that jointly guarantee ($\delta,\epsilon$)-completeness. Second, if a solution is not found using a \decomp pair $(\X,r)$, then no $\delta$-clear solution exists. That is, ($\delta,\epsilon$)-completeness can be used for deterministic infeasibility proofs~\cite{li2023sampling}. Third, the computational complexity of constructing a PRM graph can be tuned according to the desired values of $\delta$ and $\epsilon$. 



\subsection{Problem definition}
A $(\delta,\epsilon)$-complete sample set $\X$ and radius $r>0$ can be obtained by constructing a so-called $\beta$-cover~\cite{dayan2023near}. 

\begin{definition}
For a given $\beta>0$, a sample set $\X\subset \dR^d$ is a $\beta$-cover if for every point $p\in \C=\dR^d$ there exists a sample $x \in \X$ such that $\|p-x\|\leq \beta$.
\end{definition}

Note that the coverage property above is defined with respect to the whole configuration space, rather than a specific free space. The connection between $(\delta,\epsilon)$-completeness and $\beta$-cover is established in the following lemma.
\begin{lemma}[Completeness-cover relation~\cite{tsao2020sample}]\label{lem:cover}
  Fix $\delta >0$ and $\epsilon>0$. Suppose that a sample set $\X$ is a ${\beta^*}$-cover, where
  \begin{equation}
    {\beta^*}={\beta^*}(\delta,\epsilon):=\frac{\delta\epsilon}{\sqrt{1+\epsilon^2}}.
  \end{equation}
  Then $(\X,{r^*})$ is $(\delta,\epsilon)$-complete, where
  \begin{equation}
      {r^*}={r^*}(\delta,\epsilon):=\frac{2\delta(1+\epsilon)}{\sqrt{1+\epsilon^2}}.
  \end{equation}
\end{lemma}

The lemma prescribes an approach for constructing sample sets and connection radii that satisfy the $(\delta,\epsilon)$-completeness requirement. 
Considering that we will fix the radius $r^*:=r^*(\delta,\epsilon)$ throughout this work, we will say that a sample set $\X$ is $(\delta,\epsilon)$-complete if the pair $(\X,r^*)$ is $(\delta,\epsilon)$-complete.

\Cref{lem:cover} still leaves a critical unresolved question: How do we find a \decomp sample $\X$ minimizing the computation time of the PRM graph and subsequent methods? In this work, we are interested in the following two related problems. The first deals with finding a sample set of minimal size, which can be used as a proxy for computation time. 

\begin{problem}[Sample complexity]\label{problem:sample}
  For a given $\delta>0,\epsilon>0$, find a $(\delta,\epsilon)$-complete sample set $\X$ of  minimal \emph{sample complexity}, i.e., $|\X\cap \B_{r^*}|$. 
\end{problem}

Previous work~\cite{tsao2020sample,dayan2023near} has considered a slightly different notion for sample complexity aiming to minimize the global expression $|\X\cap \C|$ with a bounded $\C$. We believe that our local notion, within an ${r^*}$-ball, is more informative and optimistic, as typical problems have obstacles, and the solution lies in a small subset of the search space. Furthermore, it would allow us to obtain tighter analyses by exploiting methods from discrete geometry that reason about ball structures.

Previous work has introduced several sample sets and analyzed their sample complexity~\cite{tsao2020sample,dayan2023near}. In this work, we consider more compact sets. Moreover, we introduce a new notion that better captures the computational complexity of constructing a PRM graph. In particular, we leverage the observation that the computational complexity of sampling-based planning is typically dominated by the amount of collision checks performed~\cite{KleinbortSH16}. Furthermore, nearest-neighbor search, which is another key contributor to the algorithm's computational complexity, 
can be eliminated for deterministic samples, assuming that they have a regular structure (as for lattice-based samples, which we describe below).

Collision checks are run both on the PRM vertices and edges, where edge checks are usually performed via dense sampling of configurations along the edges and individually validating each configuration. Thus, the total number of collision checks is proportional to the total edge length of the graph. We use this observation to develop a more accurate proxy for computational complexity. In particular, we will estimate the length of edges adjacent to the origin point $o \in \dR^d$, which is a vertex in all the sample sets introduced below. Moreover, due to the regularity of the sets, the attribute below is equal across all vertices (in the absence of obstacles).

\begin{problem}[Collision-check complexity]\label{problem:collision}
  For given \mbox{$\delta>0,\epsilon>0$}, find a $(\delta,\epsilon)$-complete sample set $\X$ of  minimal \emph{collision-check complexity}, i.e., minimizing the expression
  \[
      CC_\X:=\sum_{x\in \X\cap \B_{r^*}} \|x\|.
  \]
\end{problem}

\section{Lattice-based sample sets}\label{sec:lattices}
We derive sample sets optimizing sample complexity (Problem~\ref{problem:sample}) and collision-check complexity (Problem~\ref{problem:collision}) both in theory and experiments. We focus on sample sets induced by lattices. 

A lattice is a point set in Euclidean space with a regular structure~\cite{conway2013sphere}.

\begin{definition}[Lattice]\label{definition:lattice}
  A lattice $\Lambda$ is defined as all the linear combinations (with integer coefficients) of a basis\footnote{A basis can be of full rank ($m=N$) or subdimensional ($m<N$). A basis can be non-unique.} $E_\Lambda=\{e_i\in \dR^N\}_{i=1}^m$ of rank $1\leq m\leq N$, i.e.,
  \[\Lambda:=\left\{\sum_{i=1}^m a_i e_i\middle| a_i\in\mathbb{Z},e_i\in E_\Lambda\right\}.\]
\end{definition}

It would be convenient to view lattices through their generator matrices. 
\begin{definition}[Lattice generator]\label{definition:generator}
    The generator matrix $G_\Lambda$ of a lattice $\Lambda$ with basis $E_\Lambda=\{e_i\in \dR^N\}_{i=1}^m$ is an $m\times N$ matrix such that for every  $1\leq i\leq m$, the row $i$ is equal to $e_i$. Note that 
\(\Lambda=\left\{a\cdot G_\Lambda\middle| a\in\dZ^{1\times m}\right\}.\)
Additionally, define $\det(\Lambda):=\det(G_{\Lambda}G_{\Lambda}^t)$.
\end{definition}

\subsection{Useful lattices}
We describe three lattices, visualized in~\Cref{fig:2d_lattices,fig:3d_lattices}. The first lattice is a simple rectangular grid, which is provided to benchmark more complicated and efficient lattices. Below, we fix the dimension $d\geq 2$.

\begin{definition}[$\dZ^d$ lattice]
   The $\dZ^d$ lattice is defined by the identity generator matrix $I\in \dR^{d\times d}$, with $\det(\ZN)=1$~\cite[p.~106]{conway2013sphere}.
\end{definition}


More efficient sample sets can be generated via the $D_d^*$ lattice~\cite[p120]{conway2013sphere}. This lattice was also presented in~\cite{dayan2023near}, where it was called a ``staggered grid''. In this work, we provide improved sample complexity bounds for lattice-based sample sets (including for the $D_d^*$ lattice 
and the $A_d^*$ lattice defined later on), following~\cite{conway2013sphere}, and develop theoretical bounds for collision-check complexity.

\begin{definition}[$D_d^*$ lattice]
   The $D_d^*$ lattice is defined by the  generator matrix
   \begin{align*}
     G_{\DN}=
     \begin{pmatrix}
         1 & 0 & 0 & \dots & 0 & 0 \\
         0 & 1 & 0 & \dots & 0 & 0 \\
         \vdots & \vdots & \vdots & \ddots & \vdots & 0 \\
         0 & 0 & 0 & \dots & 1 & 0 \\
         \frac{1}{2} & \frac{1}{2} & \frac{1}{2} & \dots & \frac{1}{2} & \frac{1}{2}
     \end{pmatrix}\in \dR^{d\times d}, 
 \end{align*}
 with $\det(\DN)=\frac{1}{4}$~\cite[p.~120]{conway2013sphere}.
 \end{definition}
 

\begin{figure}[t]
  \centering
  \subfloat[$\X_{\dZ_3}^{\delta,\epsilon}$ sample set.]{
    \includegraphics[width=0.46\columnwidth,clip]{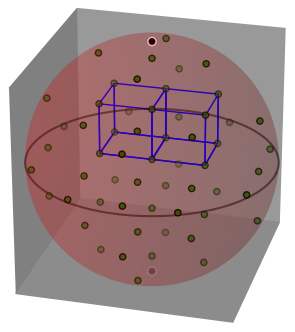}
    }
  \hfill
  \subfloat[$\X_{D_3^*}^{\delta,\epsilon}=\X_{A_3^*}^{\delta,\epsilon}$ sample sets.]{
    \includegraphics[width=0.46\columnwidth,clip]{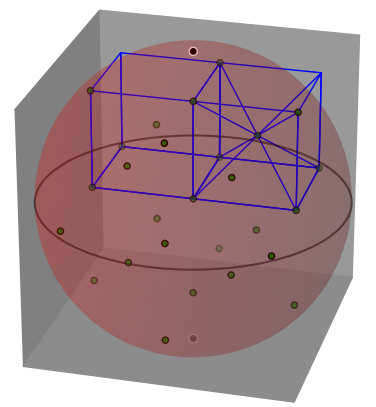}
    }
  \caption{\decomp sample sets in $\dR^3$ derived from the lattices $\dZ^3, D_3^*$ and $A^*_3$. Note the sets $\X_{D_d^*}^{\delta,\epsilon},\X_{A_d^*}^{\delta,\epsilon}$ coincide for $d=3$, and diverge for $d\geq 4$.
Note that the density of  $\X_{D^*_3}^{\delta,\eps}$ and $\X_{A^*_3}^{\delta,\eps}$ (also known as the Body-Centered Cubic structure in crystallography), and is lower than the density of $\X_{\dZ^3}^{\delta,\eps}$.}
  \label{fig:3d_lattices}
\end{figure}

The following $A^*_d$ lattice~\cite[p115]{conway2013sphere} leads to even more efficient sample sets. This lattice is also called a "hexagonal grid", and was previously used for 2D path planning~\cite{BAILEY2021103560,TengEA17}. This work is the first to consider its application in dimensions $d\geq 3$, and moreover, in the context of \decomps guarantees.


\begin{definition}[$A^*_d$ lattice]
  The $A^*_d$ lattice is defined through the generator matrix \begin{align*}
    G_{\AN}=
    \begin{pmatrix}
        1 & -1 & 0 & 0 & \dots & 0 & 0 \\
        1 & 0 & -1 & 0 & \dots & 0 & 0 \\
        \vdots & \vdots & \vdots & \vdots & \ddots & \vdots & 0 \\
        1 & 0 & 0 & 0 & \dots & -1 & 0 \\
        \frac{-d}{d+1} & \frac{1}{d+1} & \frac{1}{d+1} & \frac{1}{d+1} & \dots & \frac{1}{d+1} & \frac{1}{d+1}
    \end{pmatrix}_{d\times (d+1)}\!\!\!,
\end{align*}
with $\det(\AN)=\frac{1}{\sqrt{d+1}}$~\cite[p.~115]{conway2013sphere}.
\end{definition}

Note that $A^*_d$ is contained in $\mathbb{R}^{d+1}$ (due to the number of rows of the generator matrix), but the lattice itself is $d$-dimensional as it lies in a $d$-dimensional hyperplane (for any lattice point $(x_1,\ldots x_{d+1})\in A^*_d$ it holds that  $\sum_{i=1}^{d+1} x_i=0$). 
Our motivation for considering $\AN$ is its low \emph{density}, defined as the average number of spheres (centered on lattice points)  containing a point of the space~\cite{conway2013sphere}. In particular, $\AN$ is the best lattice covering (and best covering in general) in terms of density for dimension $d\leq 5$ (see~\cite{ryshkov1975solution}) and overall the best \emph{known} covering for $d\leq 21$. 

\subsection{From lattices to \decomp sample sets}
We derive sample sets from the lattices above by transforming the lattices such that the resulting point sets lie in $\dR^d$ and form $\beta^*$-covers (Lemma~\ref{lem:cover}). To achieve that, we will leverage the geometry of the lattices and their covering radius, which is defined below.

\begin{definition} (Covering radius~\cite{conway2013sphere})
   For a point set $\X\subset \dR^d$, a covering radius is defined to be
    \[
        f_{\X} = \sup_{y\in\dR^d}\inf_{x\in\X} \|x-y\|.
    \]
\end{definition}
When considering the covering radius of $\AN$, which lies in $\dR^{d+1}$, we will abuse the above definition to refer to its covering radius in the $d$-dimensional plane $\sum_{i=1}^{d+1}x_{d+1}=0$. Note that in order to cover $\dR^d$ with balls of radius $\rho>0$ centered at the points of a set $\X$, it must hold that $\rho\geq f_{\X}$.


\begin{thm} \label{thm:decomp_lattices}
  Fix $\delta>0,\epsilon>0$, and take $\beta^*$ as defined in \Cref{lem:cover}. Then the following sample sets are \decomp:
  \begin{enumerate}
      \item $\XZ:=\left\{\frac{2\beta^*}{\sqrt{d}}\cdot v,v\in\mathbb{Z}^d\right \}=\frac{2\beta^*}{\sqrt{d}}\cdot \ZN$.
      \item $\XD:=\\
      \begin{cases}
        d\text{ is odd: } \left\{\frac{4\beta^*}{\sqrt{2d-1}}G_{\DN}^t\cdot v,v\in\mathbb{Z}^d\right \}=\frac{4\beta^*}{\sqrt{2d-1}}\cdot \DN,\\
        d\text{ is even: }
        \left\{\sqrt{\frac{8}{d}}\beta^* G_{\DN}^t\cdot v,v\in\mathbb{Z}^d\right\}=\sqrt{\frac{8}{d}}\beta^*\cdot \DN.
      \end{cases}$
      \item $\XA:=\left\{\sqrt{\frac{12\left(d+1\right)}{d\left(d+2\right)}}\beta^* T\cdot v,v\in\mathbb{Z}^d\right\}$, 
       where   \\$T:=\begin{pmatrix}
                    1 &  1  & \dots & 1 & a - 1\\
                    -1 & 0  & \dots & 0 & a \\
                    0 & -1  & \dots & 0 & a \\
                    \vdots & \vdots  &  \ddots & \vdots & \vdots \\
                    0 & 0  &  \dots & -1 & a \\
                \end{pmatrix}\in \dR^{d \times d}$ and $a=\frac{1}{d+1 - \sqrt{d+1}}$. 
  \end{enumerate}
\end{thm}

Each of the new sample sets can be viewed as a lattice in $\dR^d$, according to Definition~\ref{definition:lattice}. For instance, $\XA$ is a lattice with the generator matrix $\sqrt{\frac{12\left(d+1\right)}{d\left(d+2\right)}}\beta^* T^t$. Also, note that the result above for $\XD$ is a tightening 
of the result in~\cite{dayan2023near} for odd dimensions.

\begin{proof}
    The lattices \Lattices require a transformation to achieve $(\de)$-completeness for a given value of $\delta$ and $\epsilon$. For a given lattice $\Lambda$, we compute a rescaling factor $w_\Lambda>0$ such that the covering radius of the lattice $w_\Lambda \Lambda$ is not bigger than $\beta^*$. This would imply that $w_\Lambda \Lambda$ is \decomp according to Lemma~\ref{lem:cover}. In particular, $w_\Lambda = \beta^*f_\Lambda^{-1}$ where $f_\Lambda$ is the covering radius of $\Lambda$. Next, we consider each of the three lattices individually.     
    

\begin{figure}[!h]
  \centering
\includegraphics[width=0.7\columnwidth]{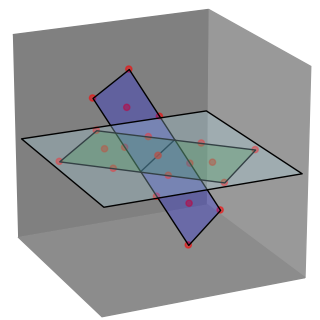}
  \caption{Visualization of embedding the lattice  $A_2^*$ originally defined in $\dR^3$ onto $\dR^2$ via the mapping $T$. The blue rectangle represents the plane $H$, where the corresponding $A_2^*$ lattice points are drawn in red. The points are generated by taking integer vectors in $\dR^d$ and applying the mapping $G^t$.  $H$ and $A_2^*$ is reflected onto the plane $H_0=\{x_3=0\}$ using the mapping $PG^t$ (denoted by the green rectangle). The third dimension is removed via the mapping $E$ to yield the embedding of $A_2^*$ in $\dR^2$.}
  \label{fig:egpt_visual}
\end{figure}

\vspace{5pt}
\noindent\emph{The lattice $\dZ^d$.} This lattice can be viewed as a set of axis-aligned unit hypercubes whose vertices are the lattice points. The center point of a cube is located at a distance of $\sqrt{d}/2$ from the cube's vertices, which yields the covering radius $f_{\ZN} =\sqrt{d}/2$, and the rescaling factor $w_{\ZN} = {\beta^*}/f_{\ZN}=2{\beta^*}/\sqrt{d}$. As a result, the sample set $\XZ:=\frac{2{\beta^*}}{\sqrt{d}}\ZN$, is a ${\beta^*}$-cover, and \decomp due to Lemma~\ref{lem:cover}. 

    
\vspace{5pt}
\noindent\emph{The lattice $\DN$.}
We use the covering radius of $\DN$, which depends on whether the dimension $d$ is odd or even~\cite[page 120]{conway2013sphere}. In particular, 
$f_{D_{d\text{-odd}}^*}=\frac{\sqrt{2d-1}}{4}$, and $f_{D_{d\text{-even}}^*}=\frac{\sqrt{2d}}{4}$. This immediately implies the definition of $\XD$.


\vspace{5pt}
\noindent\emph{The lattice $\AN$.} 
    Recall that $A_d^*$ has the generator matrix $G:=G_{\AN}\in \dR^{d\times (d+1)}$.
    As this is a mapping from $d$ to $d+1$, we start with the process of embedding $A_d^*$ in $\dR^d$ (see Figure~\ref{fig:egpt_visual}). Afterward, we show that the embedding in $\dR^d$ shares the same covering radius as the original set in $\dR^{d+1}$.

    Any row $i$ of $G$, denoted by  $G_i:=(g_{i1},g_{i2},\ldots,g_{i(d+1)})$, lies on the (hyper)plane $H: \{\sum_{j=1}^{d+1} g_{ij} = 0\}$. Thus, $A_d^*$ itself is contained in that $d$-dimensional plane. It remains to find a transformation of $\AN$ such that the dimension is reduced to $d$ while maintaining the structure of the points in $\AN$.

In the first step, we reflect $\AN$ lattice points onto the plane $H_0:=\{x_{d+1}=0\}$ using the Householder matrix 
\begin{align}\label{eq:reflection}
 P:=\begin{pNiceArray}{cw{c}{1cm}c|c}[margin]
            \Block{3-3}<\Large>{I_d - \frac{1}{D-\sqrt{D}}\mathds{1}}
            & & & \tfrac{1}{\sqrt{D}} \\
            & & & \Vdots \\
            & & & \tfrac{1}{\sqrt{D}} \\
            \hline
            \tfrac{1}{\sqrt{D}} & \dots& \tfrac{1}{\sqrt{D}} & \tfrac{1}{\sqrt{D}}
        \end{pNiceArray},
    \end{align}
where $D:=d+1$, $I_d$ is an $d\times d$ identity matrix, and $\mathds{1}$ is the $d\times d$ matrix with $1$s in all its entries.  That is, we reflect a lattice point $p\in \dR^{d+1}$ by computing the value ${v}=P\cdot p$. (See the derivation of $P$ in \conditionaltext{\Cref{app:decomp_lattice_proof}}{ the supplementary material}.)

It remains to eliminate the $(d+1)$th dimension of the points $v:=P\cdot p$. This is accomplished by the mapping
\begin{align*}
        E=
        \begin{pmatrix}
            1 & 0 & \dots & 0 & 0 \\
            0 & 1 & \dots & 0 & 0 \\
            \vdots & \vdots & \ddots & \vdots & 0 \\
            0 & 0 & \dots & 1 & 0
        \end{pmatrix}_{d\times(d+1)}.
    \end{align*}
    
We finish by computing the an explicit mapping 
that yields the embedding of $A_d^*$ to $\dR^d$. In particular, we have the embedding $T(g):=EPG^t(g)$, for $g\in \dZ^d$, where $T$ is as specified in the statement of this theorem. (See the derivation of $T$ in  \conditionaltext{\Cref{app:decomp_lattice_proof}}{ in the supplementary material}.)


For the final part, we wish to derive an \decomp sample set $\XA$. Here, we first recall that the covering radius $f_{\AN}$ is equal to $\sqrt{\frac{d\left(d+2\right)}{12\left(d+1\right)}}$~\cite[page 115]{conway2013sphere}. Considering that reflections and embeddings are isometries, i.e.,  they preserve distances between pairs of points, we can use the same covering radius after mapping the points of $\AN$ using $T$. Thus, we obtain the rescaling coefficient $w_{\AN}:={\beta^*} f_{\AN}^{-1}$, which implies that $\XA:=
         \{w_{\AN}T\cdot v,v\in\mathbb{Z}^d\}$ is \decomp.

\end{proof}

\section{Sample complexity}\label{sec:sample_complexity}
We derive the following lower and upper bounds on the sample complexity of the sets $\XZ,\XD$, and $\XA$. 

\begin{thm}[Sample-complexity bounds]
\label{thm:general_sample_complexity}
    Consider a lattice $\Lambda\in \{\dZ^d,D_d^*,A_d^*\}$ with a covering radius $f_\Lambda$, which yields the \decomp set $\XL$ for some $\delta>0,\epsilon>0$. 
    Then,
\begin{align}\label{eq:sample bounds}
        |\XL\cap \B_{r^*}|= 
        \frac{\partial(\B_1)}{\sqrt{\det(\Lambda)}}\btheta_{r^*}^d + P_d(\btheta_{r^*}),
    \end{align}
where $\btheta_{r^*}:=2f_\Lambda\left(1+\frac{1}{\epsilon}\right)$, and $P_d(\alpha)\in \dR$ is the discrepancy function~\cite{ivic2004lattice}.\footnote{For a value $\alpha>0$, $|P_3(\alpha)|=\Omega_+(\sqrt{\alpha\log(\alpha)})$ and $|P_3(\alpha)|=O\left(\alpha^{\frac{21}{32}+\epsilon}\right)$, $|P_4(\alpha)|=\Omega(\alpha^{2})$, $|P_4(\alpha)|=O\left(\alpha\log^{2/3}\alpha\right)$, and $|P_d(\alpha)|=\Theta(\alpha^{d-2})$ for $d>4$. Notice that $P_d$ can be negative.  For two functions $f,g$, the notation $f(\alpha)=O(g(\alpha))$ (or $f(\alpha)=\Omega(g(\alpha))$)  means that there exists a constant $m_u>0$ (or $m_l>0$) such that for a large enough $\alpha$ it holds that $f(\alpha)\leq m_u(g(\alpha))$ (or $f(\alpha)\geq m_l(g(\alpha))$). 
}
\end{thm}

\begin{proof}
  We estimate the size of the sample set $\XL$ induced by the lattice $\Lambda$ and the scaling factor $f_\Lambda$ by exploiting the relation between $\XL$ and the grid lattice~$\dZ^d$. In particular, a ball with respect to $\XL$ can be viewed as a rescaled ball for the  $\dZ^d$ lattice. This allows the use of bounds on the number of $\dZ^d$ points within an ellipse. 

  Fix a ball radius $R>0$. 
  Due to the rescaling performed in Theorem~\ref{thm:decomp_lattices}  we transition from the lattice $\Lambda$, which is a $f_\Lambda$-cover for $\dR^d$, into the set $\XL$, which is a ${\beta^*}$-cover for $\dR^d$. In particular, the rescaling factor is $w_\Lambda:={\beta^*}/f_\Lambda$, which is multiplied by $\Lambda$ to obtain $\XL$ (for $\AN$ we also applied an isometric transformation, but this does not change the scale reasoning). Thus, the ball $\B_R$ with respect to $\XL$ can be viewed as the ball $\B_{\btheta}$, where $\btheta_R=R w_\Lambda=\frac{R}{{\beta^*}}f_\Lambda$. Thus, $|\XL\cap \B_R|=|\Lambda\cap\B_{\btheta_R}|$. For the remainder of the proof, we wish to bound the expression $|\Lambda\cap\B_{\btheta_R}|$.

  Let $v=aG_\Lambda$ be a lattice $\Lambda$ point, where $a\in\mathbb{Z}^d$. By definition of the \emph{Gram matrix} $\text{Gr}_\Lambda:=G_\Lambda G_\Lambda^t$, we obtain 
  \[
    a\text{Gr}_\Lambda a^t:=aG_\Lambda G^t_\Lambda a^t=\|aG_\Lambda\|^2=\|v\|^2.
  \]
This leads to the relation
  \begin{align*}
      \B_{\btheta_R}\cap \Lambda &= \{v\in\Lambda |\,\|v\| \leq \btheta_R\} =\{a\in\dZ^d|a \text{Gr}_\Lambda a^t\leq \btheta_R^2\} \\ & =E_{\btheta_R^2}\left(\text{Gr}_\Lambda\right)\cap \dZ^d,
  \end{align*}
  where $E_{s}\left(A\right):=\{x\in\dR^d|x A x^t\leq s\}$ for a matrix $A$ and $s>0$. 
  
  Next, observe that the Gram matrix $Gr_\Lambda$ of the lattices \Lattices contains only rational numbers, and hence defines a \emph{rational quadratic form}, which allows us employ \emph{rational ellipsoid bounds}~\cite{ivic2004lattice} for $|E_{\btheta_R}\left(\text{Gr}_\Lambda\right)\cap \dZ^d|$. Hence,
\begin{align}
    |\XL\cap \B_R
|&=|\B_{\btheta_R}\cap \Lambda|=\left|E_{\btheta_R}\left(\text{Gr}_\Lambda\right)\cap \dZ^d\right|\nonumber\\
&=\frac{\partial(\B_{\btheta_R})}{\sqrt{\det(\Lambda)}}+P_d(\btheta_R)=\frac{\partial(\B_1)}{\sqrt{\det(\Lambda)}}\btheta_R^d+P_d(\btheta_R). \label{eq:sample_complexity}
\end{align}
The expression in Equation~\eqref{eq:sample bounds} immediately follows by plugging
\[\btheta_{r^*}=\frac{r^*}{{\beta^*}}f_\Lambda = \frac{2\delta(1+\epsilon)}{\sqrt{1+\epsilon^2}}\cdot \frac{\sqrt{1+\epsilon^2}}{\delta\epsilon} f_\Lambda = \frac{2(1+\eps)}{\eps}f_\Lambda.\]
\end{proof}

\niceparagraph{Discussion.} 
The expression in Equation~\eqref{eq:sample bounds} is identical for our three sample sets, except for the value of the covering radius $f_\Lambda$. This highlights the fact that a smaller covering radii leads to a lower sample complexity. Also, notice that the value $f_\Lambda$ is raised to the power of $d$ in Equation~\eqref{eq:sample bounds}, which emphasizes the difference between the sets in terms of sample complexity. 
See a plot of the sample complexity in theory and practice in~\Cref{fig:limit_graph_upper}, wherein $\XA$ has the lowest sample complexity (except in dimensions $3$ where it coincides with $\XD$). Notice that the theoretical bounds are well-aligned with the practical values. 

For example, consider $d=12$, where the $\XA$ sample set is  $\approx 4$ times smaller than $\XD$.  Even when this difference is not as big, e.g., in dimension $6$ where the size of $\XA$ is  $\approx 1.63$ times smaller than $\XD$, we observe tremendous impact in terms of the running of the motion-planning algorithm in experiments. This follows from the fact that the sample complexity studied here corresponds to the branching factor of the underlying search algorithm, which is known to have a significant impact on the running time. In the next session, we show that the superiority of $\XA$ is maintained also for the collision-check complexity metric.

\begin{figure}[thb]
\includegraphics[width=\columnwidth]{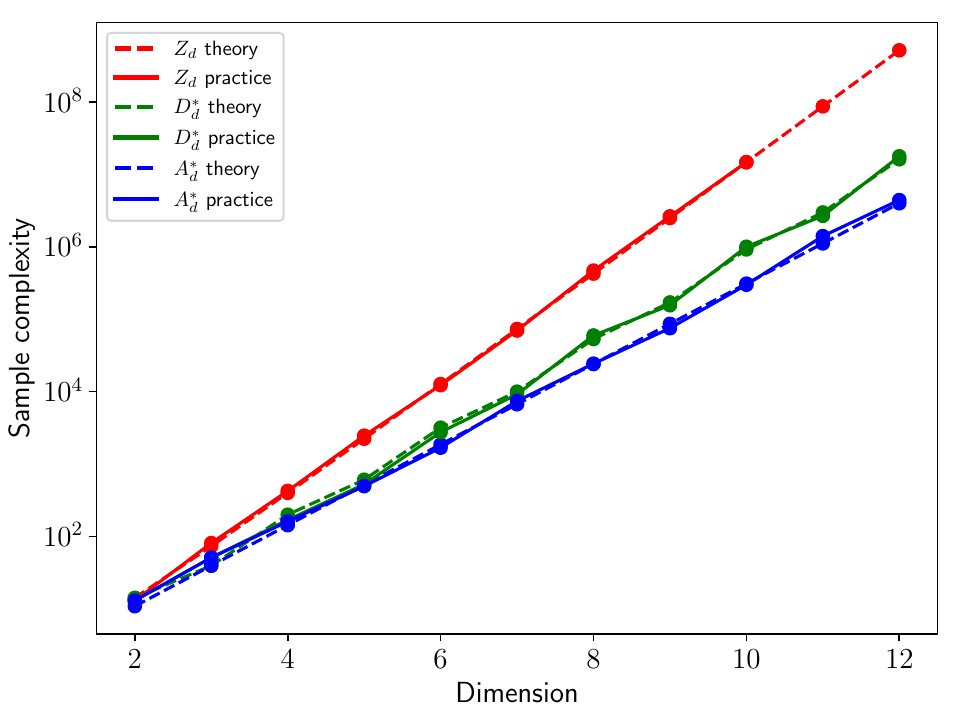}
\caption{A sample-complexity  plot for the sample sets $\XZ,\XD$, and $\XA$ with $\delta=1$ and $\eps=2$. The dashed line represents the theoretical approximation (\Cref{eq:sample bounds}), where the asymptotic error term $P_d$ is excluded. The solid line depicts the practical value, i.e., the  number of lattice points within the $r^*$-ball in practice. Missing values are due to memory limitations.}
\label{fig:limit_graph_upper}
\end{figure}

\section{Collision-check complexity}\label{sec:collision_complexity}
In this section, we seek to analyze the collision-check complexity of our sample sets $\XZ,\XD,\XA$, as the length of the graph edges is a better proxy for the computational complexity of constructing the resulting PRM than sample complexity. 

Recall that the collision-check complexity of a lattice-based sample set $\X_\Lambda$  is 
$CC_{\X_\Lambda}=\sum_{x\in \X_\Lambda\cap \B_{r^*}}\|x\|$.
A naive approach to upper-bound $CC_{\X_\Lambda}$ is to multiply the number of points in a $r^*$-ball (which corresponds to the sample complexity of $\X_\Lambda$) with the radius $r^*$, i.e.,
\begin{align}
    CC_{\X_\Lambda}\leq r^*\cdot |X_\Lambda \cap \B_{r^*}|=r^*\cdot \frac{\partial(B_1)}{\sqrt{\det(\Lambda)}}\btheta^d_{r^*} + r^*\cdot P_d(\btheta_{r^*}). \label{eq:cc_naive}
\end{align}
The following is a tighter bound, which reduces the coefficient of the dominant factor of the naive bound.

\begin{thm}[CC complexity bound]\label{thm:CC}
  Consider a lattice $\Lambda\in \{\dZ^d,D_d^*,A_d^*\}$ with a covering radius $f_\Lambda$, which yields the \decomp set $\X:=\XL$ for some $\delta>0,\epsilon>0$ and dimension $d\geq 2$. 
  Then,
    \begin{align}
        CC_\X\leq \zeta \cdot r^* \cdot \frac{\partial(B_1)}{\sqrt{\det(\Lambda)}}  \btheta_{r^*}^d + r^* \cdot P_d(\btheta_{r^*}), \label{eq:CC_thm}
    \end{align}
    where $\btheta_{r^*}:=2f_\Lambda\left(1+\frac{1}{\epsilon}\right),\zeta:=\left(1 - \frac{\xi^{d+2} - \xi}{ d\xi - (d+1)}\right)<1$ and     
    $\xi:=\left(\frac{d}{d+1}\right)^d$. 
\end{thm}
\begin{proof}
  We improve the naive bound in Equation~\eqref{eq:cc_naive} by partitioning the $r$-ball into annuli. Consider a sequence of ${k+1\geq 2}$  radii ${0<r_k<\ldots<r_0=r^*}$, where $r_i:=\td^i r$ and $\td:=\frac{d}{d+1}$. This leads to the bound
\begin{align}
\label{eq:cc_eval}
CC_\X&\leq \sum_{i=0}^{k-1}r_i |\X\cap (\B_{r_i}\setminus \B_{r_{i+1}})| + r_k|\X\cap \B_{r_k}|\nonumber\\
  &= \sum_{i=0}^{k-1}r_i\left(|\X\cap \B_{r_i}|- |\X\cap \B_{r_{i+1}}|\right)+ r_k|\X\cap \B_{r_k}| \nonumber\\
  &= r^*|\X\cap \B_{r}| + \sum_{i=1}^{k}(r_i - r_{i-1})|\X\cap \B_{r_i}| \nonumber\\
  & =r^*\left(\frac{\partial(B_1)}{\sqrt{\det(\Lambda)}}\btheta^d_{r} +P_d(\btheta^{d-2}_{r})\right) \nonumber\\&+ \sum_{i=1}^{k}(r_i - r_{i-1}) \left(\frac{\partial(B_1)}{\sqrt{\det(\Lambda)}}\btheta^d_{r_i} +P_d(\btheta^{d-2}_{r_i})\right) \nonumber\\
  &= \frac{\partial(B_1)}{\sqrt{\det(\Lambda)}} \left(\underbrace{r^* \btheta^d_{r^*} + \sum_{i=1}^k(r_i-r_{i-1}) \btheta^d_{r_i}}_{:=\gamma}\right) + r^* P_d(\btheta_{r^*}).
\end{align}

Next, it can be shown that  
\[\gamma= r^* \btheta^d_{r^*} \left(1 - \frac{\xi^{d+2} - \xi}{ d\xi - (d+1)}\right)=r^* \btheta^d_{r^*} \zeta,\]
where 
$\btheta_{r_i}= r_i\frac{\btheta_{r^*}}{r^*}, k=d,$ and $\xi:=\td^d=\left(\frac{d}{d+1}\right)^d$. See \conditionaltext{\Cref{app:CC}}{ the supplementary material} for the full details of this derivation, including the motivation behind the definitions of $r_i$ and $k$.

By plugging the value of $\gamma$ in Equation~\eqref{eq:cc_eval}, Equation~\eqref{eq:CC_thm} follows. 
Notice that $\td < 1$. Therefore, $\td^{d(d+1)}-1 < 0$ and $\td^{d+1} - 1 < 0$, so our expression in the parenthesis is less than 1, which implies an improvement over the trivial bound (Equation~\eqref{eq:cc_naive}). 
\end{proof}

\niceparagraph{Discussion.}
Theorem~\ref{thm:CC} leads to a constant-factor improvement over the naive bound concerning the coefficient of the main bound term.  The value of $\gamma$  ranges from $\approx 0.751$ for $d=2$ and monotonically increases towards $1$ as $d\rightarrow\infty$ (see a plot in \conditionaltext{\Cref{fig:annuli_bound:app}}{ the supplementary material}).  Although our result currently leads to modest improvement, we hope it will pave the way for tighter bounds in the future. Nevertheless, our analysis emphasizes the impact of the coverage quality of lattices on the running time, which is more substantial than what the sample-complexity bound suggests. 

See a plot of the theoretical bound, along with the practical values in Figure~\ref{fig:limit_graph_upper}, where both values show similar trends. With that said, the discrepancy between the leading value in \Cref{eq:CC_thm} and the practical value is bigger than in the sample-complexity case, which suggests that the result in Theorem~\ref{thm:CC} could be tightened. Finally, note that collision-check complexity is roughly an order of magnitude larger than the sample complexity, particularly in higher dimensions. 
 
\begin{figure}[thb]
\includegraphics[width=\columnwidth]{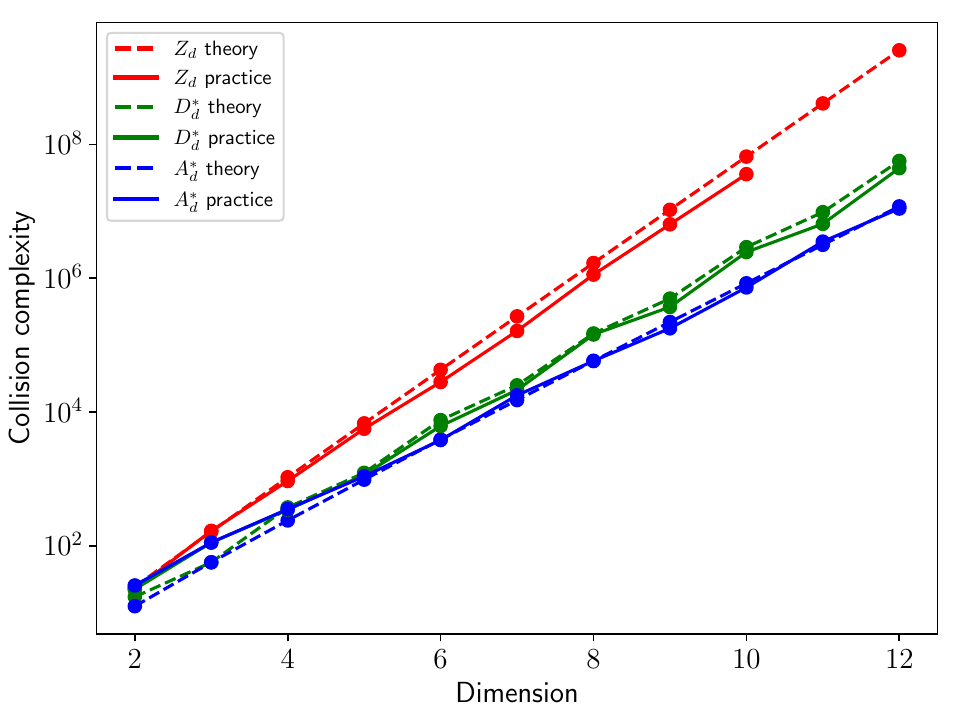}
\caption{A collision-check complexity  plot for the sample sets $\XZ,\XD$, and $\XA$ with $\delta=1$ and $\eps=2$. The dashed line represents the theoretical approximation (\Cref{eq:CC_thm}), where the asymptotic error term $P_d$ is excluded.} 
\label{fig:limit_graph_upper}
\end{figure}

\section{Algorithmic implications}
In this section, we examine the algorithmic implications of our theoretical results. We discuss the types of sampling-based planners (SBPs) that can benefit from lattice-based sampling (LBS), as defined in Theorem~\ref{thm:decomp_lattices}. We then discuss how to efficiently generate the points of a given LBS $\mathcal{X}_{\Lambda}^{\delta,\varepsilon}$. Finally, 
we provide a pseudo-code for an SBP that we use in the experimental results, while highlighting algorithmic aspects that leverage the regularity of LBS.

\subsection{Batch planners}
LBS is well suited for batch SBPs that generate a priori a batch (or several batches) of samples, implicitly or explicitly (see below), and process them simultaneously, where the resulting graph structure is agnostic to ordering between the samples, as in Equation~\eqref{eq:graph}. In our context, an LBS $\XL$ is regarded as a sample batch.  The multi-query planners PRM~\cite{kavraki1996probabilistic}, PRM*~\cite{karaman2011sampling}, and SPARS~\cite{DobsonB14}, and the single-query planners BIT*~\cite{GammellBS20}, AIT*~\cite{StrubG22}, TMIT*~\cite{ThomasonSG22}, IRIS~\cite{FuKSA23}, LRA*~\cite{MandalikaSS18}, GLS~\cite{MandalikaCSS19}, FMT*~\cite{JSCP15} and LazySP~\cite{DellinS16}, and the multi-robot planners dRRT~\cite{SoloveySH16} and dRRT*~\cite{ShomeSDHB20}, to name just a few examples, fit this description. 

LBS does not fit, in its current form, within incremental SBPs (e.g., RRT~\cite{lavalle2006planning}, RRG, or RRT*~\cite{karaman2011sampling}), which require the ability to add new samples one by one as time permits (as in, e.g., uniform random sampling or Halton sequences). The resulting graph structure in those algorithms depends not only on the physical location of the samples but also on the order in which they were processed. 
(We do believe that LBS could be adapted for {\em incremental} processing, as we discuss in the final section.) 

\subsection{Construction of lattice points}
Next, we provide details for generating the lattice $\Lambda$ points within an $R$-ball in Algorithm~\ref{alg:getNN}. Note that if $R=r^*$,  the above point set can be viewed as the neighbor set of the origin vertex of the graph $\G_{\M(\Lambda,r^*)}$ (assuming that $\B_{r^*}\subset \C_f$). To compute $\Lambda \cap R$, we run a BFS search over the integer vectors $\mathbb{Z}^d$ starting from the origin vertex, and compute for each vector $v$ the resulting lattice point through $x:=vG_{\Lambda}$ while discarding points outside of $R$, where $G_{\Lambda}$ is the generator matrix of $\Lambda$. 

\begin{algorithm}
\caption{getLatticeNeighbors($G_\Lambda,R$)}\label{alg:getNN}
  $\open=\{\zv\}$ \tcp{initiallize search with center vertex}
  $\opennext=\emptyset$  \tcp{next layer of OPEN} 
  \visited$=\emptyset$\;
  \While{!\open.\textup{empty()} $\vee$ !$\opennext$.\textup{empty()}}{
    \If{\textup{OPEN.empty()}} {
    \tcc{switch to the next layer of neighbors}
      OPEN=OPEN$_\textup{next}$\;
      OPEN$_\textup{next}=\emptyset$\;
    }
    $p=$ OPEN.pop()\;
    \If {$p\in$ \visited}{
        \Continue
    }
    VISITED.insert($p$) \;
    \tcc{iterate over basis integer vectors $\{e_1,\ldots, e_d\}$, where $e_i$ contains "$1$" in the $i$'th coordinate, and "$0$" elsewhere, in both directions}
    \For {$i\in\{1,\dots,d\},sign\in\{\pm1\}$}  { 
      $e = sign \cdot e_i$\;
      
      $p_\textup{new} =  p+ e \cdot G_\Lambda$ \tcp{obtain next lattice point}
      \If{$p_\textup{new}\notin\visited\cup\open\cup \opennext$ and $\|p\|\leq R $}
      {
OPEN$_\textup{next}$.insert($p_\textup{new}$)\;
      }
    }
  }
  \Return{\visited}
\end{algorithm}

\subsection{Implicit A*}
To demonstrate the impact of LBS in practice, we use a single-query planner where an implicitly-represented PRM graph $\G_{\M(\XL,r^*)}$ is explored using a search heuristic, similar to BIT*~\cite{GammellBS20}, and GLS~\cite{MandalikaCSS19}. Those state-of-the-art approaches are well-suited for settings where samples are generated in large batches, as LBS facilitates. We focus on the single-query setting, as it allows us to experiment with more complex problem scenarios (e.g., in terms of dimensions and tightness) than in a multi-query setting, where the entire configuration space needs to be explored, which requires additional memory and compute time. 

The planner we use, which is termed for simplicity \emph{implicit A*} (iA*), can be viewed as a simplified version of BIT* with a single sample batch searched using the A* algorithm. iA* generates a sample set $\X$ from a given sample distribution  ($\XZ,\XD,\XA$ or uniform random sampling). Instead of constructing the entire PRM graph $G:=G_{\M(\X,r)}$ resulting from $\X$ and a given radius parameter $r$, iA* constructs a partial graph $G'\subset G$ in an implicit manner, where the construction is guided by the underlying A* search. That is, when a vertex $v$ of $G$ is expanded, its neighbors $N_v$ within an $r$-neighborhood are retrieved from $\X$, and the edges between $v$ and every $u\in N_v$ are collision-checked and added to the explored portion of the graph $G'$. For LBS, we set $r:=r^*$. 

We consider two flavors of iA*. In the first flavors, 
denoted by \glo, vertex neighbors (as $N_v$ above) are retrieved by calling a \emph{global} nearest-neighbor (NN) data structure. Before starting the A* search, this data structure is initialized with the set $\X$. 

Although the benefits of LBS over randomized sampling are already apparent for the \glo flavor (especially $\XA$), as will be shown in the next section, the performance of iA* can be further improved by exploiting the \emph{local} regular structure of lattices. In the second flavor of iA*, denoted by \loc,  which only applies to lattice-based sets, vertex neighbors are efficiently retrieved without NN data structures. Given a lattice-based sample set $\X$ and a connection radius $r>0$, denote by $N_x:=\B_r(x)\cap \X$ the $r$ nearest-neighbors of a vertex $x\in \X$. Then, for another sample $x'\in \X_\Lambda$ it holds that $N_{x'}:=\B_r(x')\cap \X=(x'-x)+N_x$, i.e., $N_{x'}$ is a translation of $N_x$. Hence, the computation of the neighbor set $N_0$, discussed above, is only performed from scratch once at the beginning of the run of the search algorithm, where the $N_v$ is easily obtained from vector operations.

We provide pseudo-code for iA*-\loc in 
Algorithm~\ref{alg:ImpA}. The algorithm follows the structure of the standard A* algorithm by extracting vertices from the open list, with some small modifications. The algorithm returns a path from $q_s$ to $q_g$ on the graph $G_\M(\X,r^*)$ (or reports that none exists otherwise), where $\X:=\XL$. Without loss of generality, we assume that $q_s\in \X$, and moreover, $q_s$ is the origin. Notice that $\X$ is given to the algorithm as an implicit parameter, as is the free space $\C_f$. As mentioned earlier, the neighbor set of the origin vertex is computed once (line~\ref{alg:line:neighbor_0}), and this information is then reused to compute the neighbor set of a general vertex $q$ (line~\ref{alg:line:neighbor_q}). This is where iA*-\loc differs from iA*-\glo, where the latter computes the neighbor set by running a nearest-neighbor search algorithm, as is common in SBPs (i.e., iA*-\glo requires an explicit representation of the full sample set). 

From an implementation standpoint, we mention that additional improvements can be obtained for iA*-\loc by associating every sample $q\in \X_\Lambda$  with the integer vector $v\in \dZ^d$ such that $q:=vG_\Lambda$. This allows for efficient tracking of explored vertices. For simplicity, we omit those fine details from the pseudo-code, which can be found in our code repository.

\begin{algorithm}
\caption{iA*-\loc($\C_f,\X,r^*,q_s,q_g$)}\label{alg:ImpA}
  $g(q_s)= 0$ \tcp{cost-to-go of $q_s$}
  OPEN $= \{q_s\}$\;
  VISITED $=\{q_s\}$ \;
  PREV$[q_s]=\emptyset$ \tcp{Predecessor vertices along shortest path to $q_s$}
  $N_0\gets \textup{getLatticeNeighbors}(\X,r^*)$ \tcp{ neighbors of $q_s$} \label{alg:line:neighbor_0}
  \While{\open\xspace not empty}{
    \tcc{extract vertex $q$ minimizing $g(q)+\|q_g-q\|$}
    $q\gets \textup{OPEN.pop\_min}()$\;
    \If{$q==q_e$}{
      \Return {path reconstructed using PREV }
    }
    \tcc{compute neighbors of $q$}
    $N_q = N_0 + q$\; \label{alg:line:neighbor_q}
    \If{$\|q-q_g\| < r^*$} {
      $N_q \gets N_q \cup \{q_g\}$ \tcp{add $q_g$ to neighbor set}
    }
    \For{$n$ in $N_q$} {
    \tcc{collision-check edge from $q$ to $n$}
      \If{$\overline{qn}\not\subset \C_f$}{
        \Continue
      }
      $g_\textup{new}(n) = g(q) + \|q-n\|$ \tcp{potential cost-to-come of $n$}
      $g_\textup{cur}(n)=\infty$ \tcp{current cost-to-come of $n$} 
      
      \If{$n\in $ VISITED}{
        $g_\textup{cur}(n) = g(n)$\;
      }
      \If{$g_\textup{new}(n) < g_\textup{cur}(n)$}{
        \If{$n\notin$ VISITED}{
          VISITED.insert($n$)
        } {
          $g(n)=g_\textup{new}(n)$ ;\
        }
        $PREV[n]=q$ \;
        \If{$n\notin$ OPEN}{
          OPEN.insert$(n)$\;
        }
      }
    } 
  }
  \Return {$\emptyset$}
\end{algorithm}


\section{Experimental results}\label{sec:experiments}
We study practical aspects of our theoretical findings in motion planning for multi-robot systems and a manipulator arm. We report comparisons between the three lattice-based sets, as well as uniform random sampling. An additional experiment studying the effect of the parameters~$\delta$ and~$\eps$ is reported in\conditionaltext{ the appendix}{ the supplementary material}. 

\begin{figure*}[]
\hspace*{-0.8cm}
  \centering
\subfloat[Kenny ($\uparrow$), Narrow ($\downarrow$)]{
\includegraphics[width=0.364\columnwidth,clip]{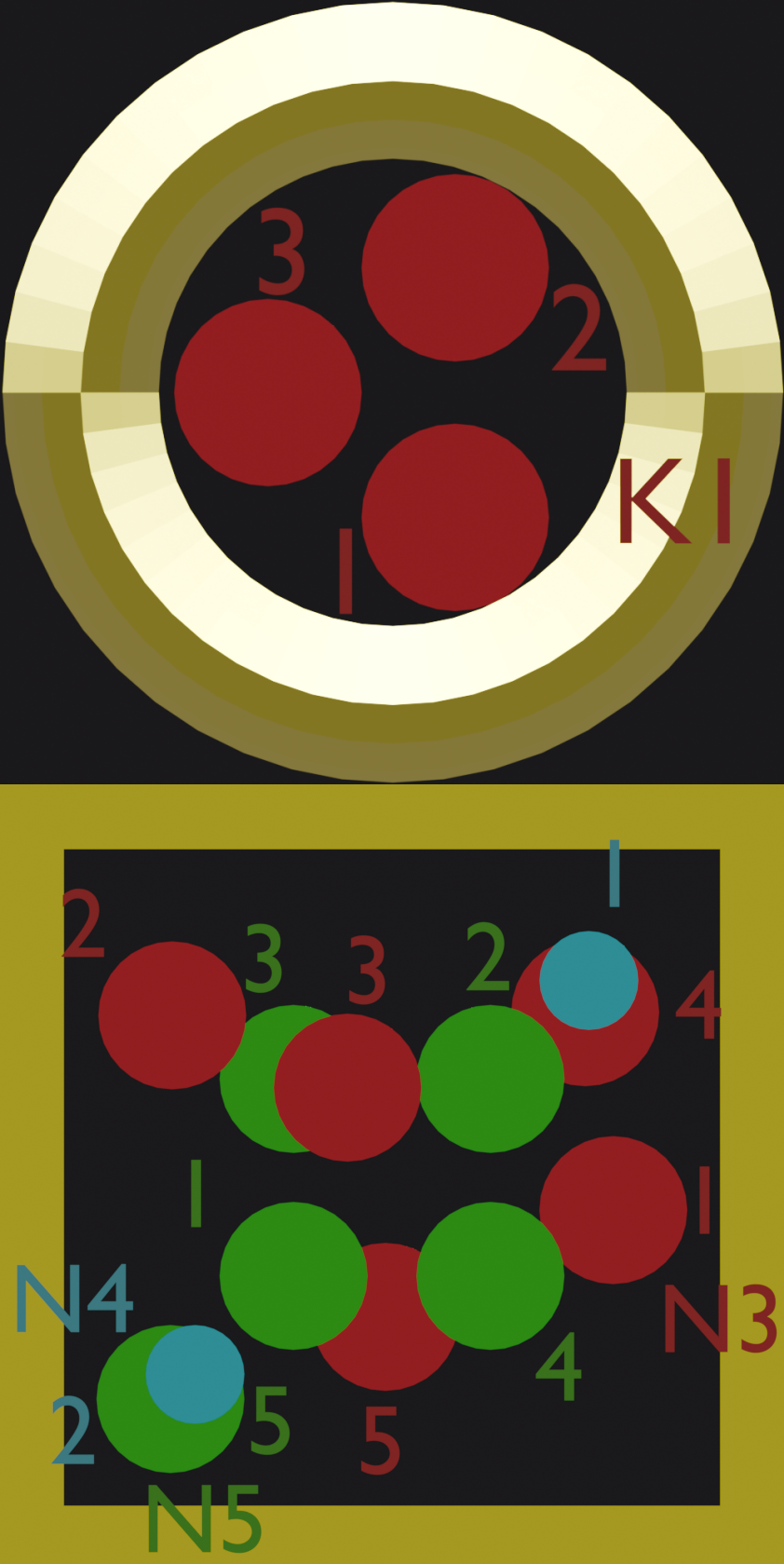}
    }
\subfloat[Zigzag]{
\includegraphics[width=0.3\columnwidth,clip]{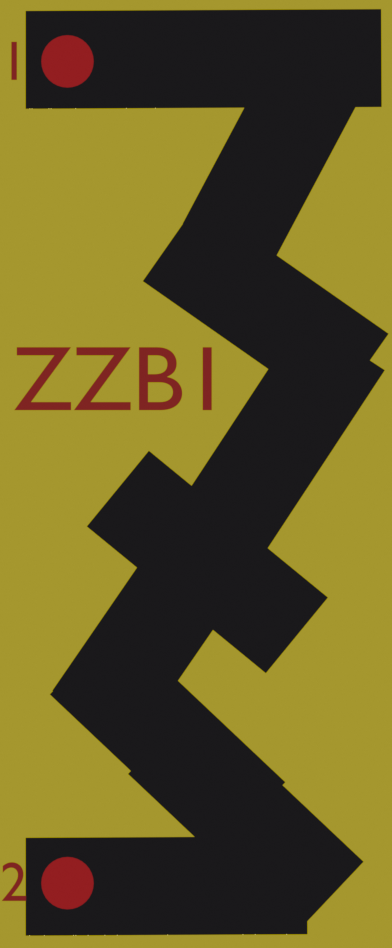}
    }
\subfloat[Unique Maze]{
\includegraphics[width=0.725\columnwidth,clip]{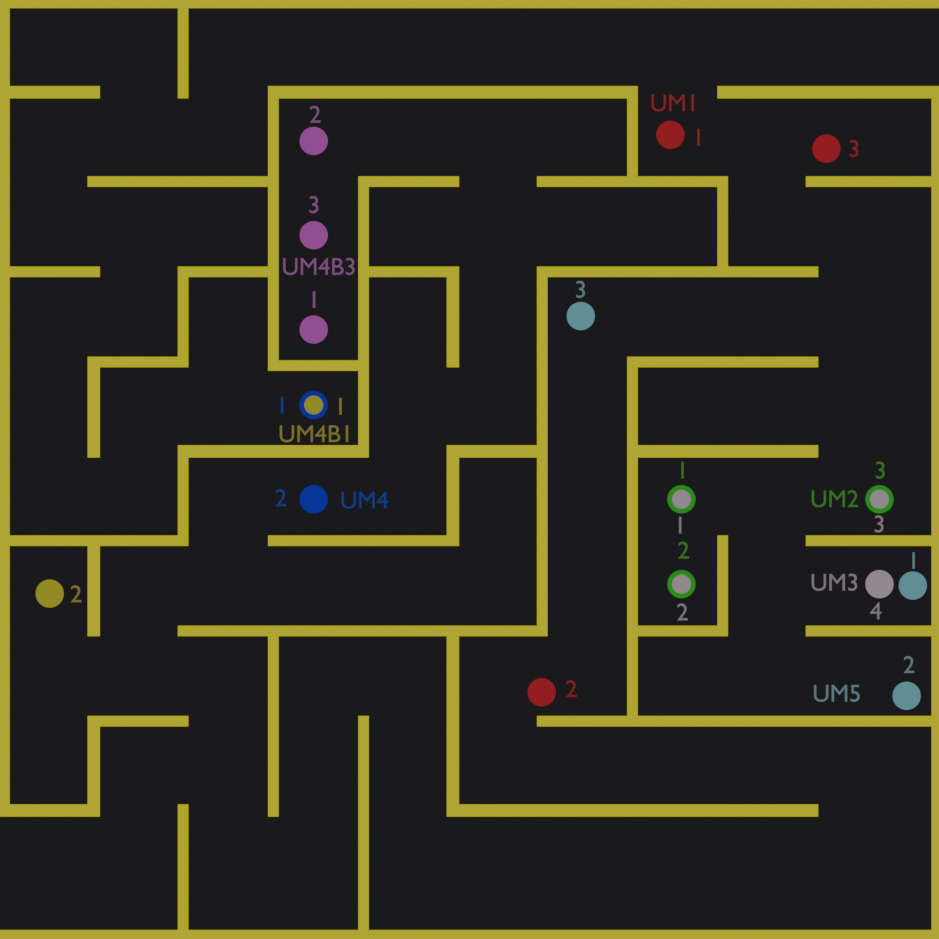}
    }
\subfloat[Bugtrap]{\includegraphics[width=0.725\columnwidth,trim={1.1cm 1.1cm 1.1cm 1.1cm},clip]{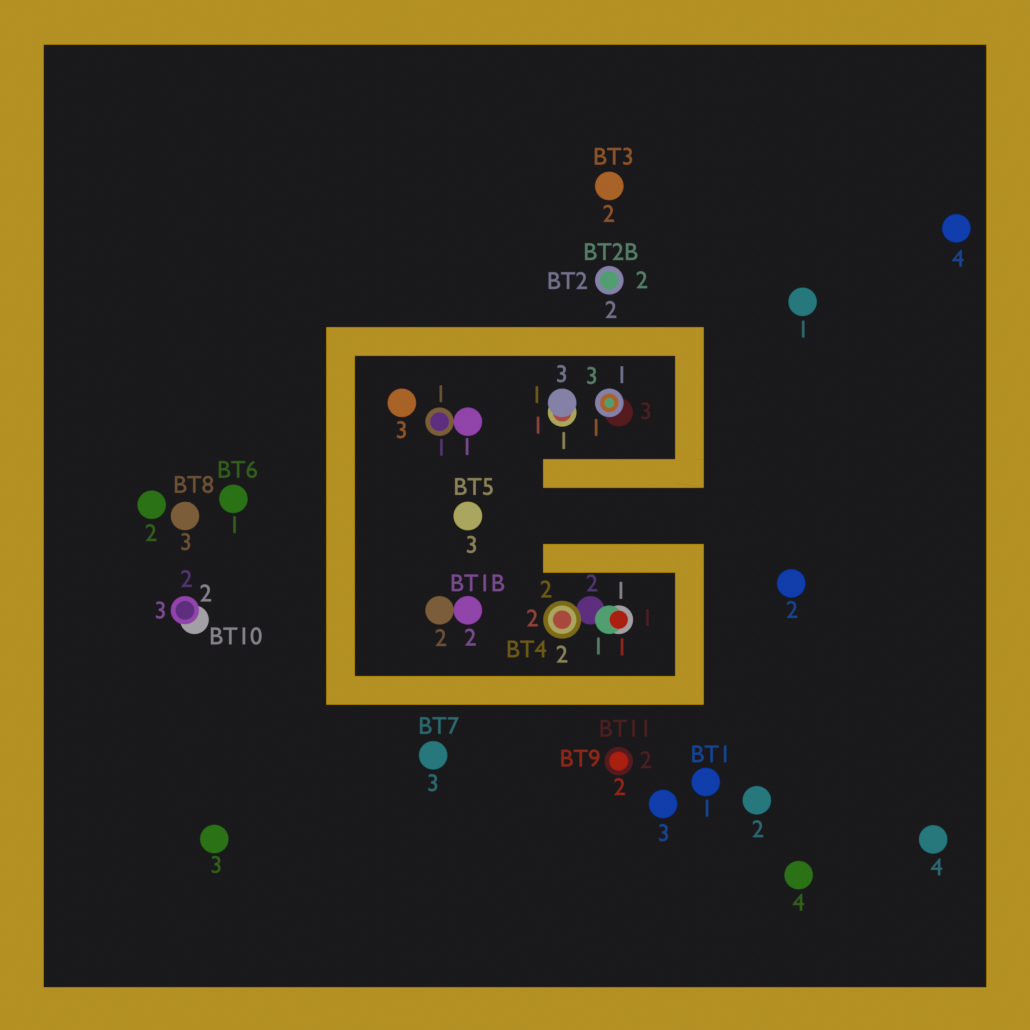}
    }

  \caption{A subset of the scenarios used in the experimental results for the multi-robot setting. Some of the figures depict several scenarios, where each scenario consists of a workspace environment, along with an initial configuration specifying the number of robots, their initial positions, and a permutation of their target positions. Each configuration is drawn in a different color, where the scenario name is indicated by the first letters of the workspace name along with a number. (a) (Top) For the "Kenny" workspace, there is a single scenario K1, where robot $1$ starts in the bottom position, robot $2$ in the top position, and robot $3$ in the remaining position, where the disc size corresponds to the robot geometry. In this map, the robots are forced to perform a simultaneous movement. (a) (Bottom)   
  For the "Narrow Room" workspace, we illustrate the scenarios N3,N4, and N5. Due to overlap, the discs in N4 are slightly shrunk for better visualizations. (b) The Zigzag scenario contains narrow, winding passages in which one of the robots can hide in a pocket to let the other robot pass. Swaps can also occur in the top or bottom corners, albeit with greater care. (c),(d) Those workspaces are taken from OMPL~\cite{sucan2012the-open-motion-planning-library}.}
  \label{fig:scenarios}
\end{figure*}

%

\subsection{Implementation details and planners}
We evaluate the performance of the iA* planner for LBS using the \loc and \glo flavors, and for uniform random sampling using the \glo variant.
The experiments were performed on an ASUS Vivobook 16x laptop equipped with an Intel Core i9-13900H CPU, 32GB DDR4 RAM, and SSD storage, running Ubuntu 22.04.5 LTS OS. 

The planners were implemented in C\texttt{++} within OMPL~\cite{sucan2012the-open-motion-planning-library}, with FCL~\cite{Pan2012FCL} for collision detection, and GNAT~\cite{gipson2013resolution} for nearest-neighbor search (where applicable). The manipulator simulations also rely on the VAMP library~\cite{vamp_2024}, which significantly speeds up collision checking via fine-grain parallelization of edge collision checking.

\subsection{Scenarios}
We test the planners on various motion-planning problems in two settings. 

\niceparagraph{Multi-robot system.} The first setting considers multi-robot systems with a total number of degrees of freedom $d\in \{4,6,8,10,12\}$, where an $m$ disc-robot system is viewed as a single-robot system whose configuration space is $\mathbb{R}^{2m}$. As such, we construct sample sets in $\mathbb{R}^{2m}$ by relying on  $2m$-dimensional lattices (or uniform random sampling). 

Each scenario consists of a multi-robot system of $m$ labeled planar disc robots that need to (simultaneously) exchange positions (i.e., robot $i\in \{0,\ldots,m-1\}$ moves to the start position of robot $i+1 \mod m$) while avoiding collisions with each other and static obstacles. 

This setting is considered for two reasons. First, such multi-robot systems can be viewed as Euclidean systems, which allows applying our theoretical results (see discussion in Section~\ref{sec:future} of extensions to more general systems). (Specifically, the configuration space of an $m$-robots system is $\dR^{2m}$.) Second, this setting allows us to test systems of various dimensions while still being able to visualize the problem setting (which is in 2D).  Third, it provides a simple approach for determining the value $\delta$ (see below).

A subset of the multi-robot test scenarios is found in Figure~\ref{fig:scenarios} (additional scenarios are found in\conditionaltext{~\Cref{fig:scenarios:app}}{ the supplementary material} along with a detailed description). The scenarios present various difficulty levels for the planners, where the most difficult scenarios consist of narrow passages for the individual robots and a significant amount of coordination between the robots, giving rise to narrow passages in the full configuration space. 

Unless stated otherwise, the parameters $\delta$ and $\eps$ were specified in the following manner. We assigned $\eps:=10$ to focus on running time scalability rather than solution quality. The parameter $\delta$ was initially set to be the clearance of the start configuration (capturing both distances from obstacles and between robots), which was decreased until a solution using $\XA$ was found. We discuss the automatic tuning of $\delta$ and $\eps$  in Section~\ref{sec:future}.

\niceparagraph{Manipulator arm}
The second setting considers a single Panda manipulator arm with $d=7$ (see scenarios in Figure~\ref{fig:manip_tests}). Here, we demonstrate the integration of lattice-based sampling within the novel VAMP library, which significantly speeds up collision checking~\cite{vamp_2024}.

Below, we evaluate several aspects of the lattice-based sampling, which we first consider in the multi-robot setting. The final part of this section considers the manipulator case.  

\subsection{Comparison between lattice-based sample sets}
In the first set of experiments, we study the running time and solution quality (in terms of path length) of the three lattice-based sample sets $\XZ,\XD$, and $\XA$ within the \loc flavor of iA* for a selected set of scenarios. Although in some scenarios, the performance between the sample sets can be comparable, in terms of running time, especially for $\XD$ and $\XA$, here we highlight situations where large gaps occur. 

The results are reported in Table~\ref{tbl:lattice_comparison}. (Results for additional scenarios are provided in\conditionaltext{~\Cref{tbl:lattice_comparison:app} in the appendix}{ the supplementary material}.) In terms of running time, $\XA$ outperforms the other sample sets, as predicted by our theoretical results with respect to sample complexity and collision-check complexity. $\XZ$ results in at least one order of magnitude (up to 2 orders) slower running times than the other two sample sets. $\XA$ outperforms $\XD$, being at least $3\times$, and sometimes as much as $10\times$, faster. The notations "dnf" and "-" indicate a running time threshold of 1000 seconds was exceeded and failure to find a solution, respectively. 

Regarding solution quality, $\XZ$ outperforms the other sample sets, except for the last scenario (unless it does not manage to find a solution) due to a denser graph. The difference in the solution length suggests that the completeness-cover relation Lemma~\ref{lem:cover} can be tightened by, e.g., reducing the value of ${\beta^*}$, albeit we emphasize this is a worst-case bound. From a practical perspective, the difference in the path length can be reduced via post-processing techniques with negligible computational cost. To summarize, $\XA$-sampling can drastically reduce computational cost while achieving comparable solution quality to the other sample distribution. For this reason, we omit comparisons with $\XZ$ and $\XD$ in the remainder of this section. 

\begin{table}[]
\caption{Comparison of running time and solution length using lattices-based sample sets (where the underlying lattice is denoted in the table) in the iA*-\loc algorithm. Solution length is normalized with respect to the length obtained using $\XA$. } \label{tbl:lattice_comparison}
\centering
\begin{tabular}{|c||ccc|cc|}
\hline
 & \multicolumn{3}{c|}{\cellcolor[HTML]{EFEFEF} Time (s)} & \multicolumn{2}{c|}{\cellcolor[HTML]{EFEFEF} Length (r)} \\ \cline{2-6} 
\multirow{-2}{*}{\begin{tabular}[c]{@{}c@{}}Scenario\\ (robot \#)\end{tabular}} & \multicolumn{1}{c|}{\cellcolor[HTML]{FFFFC7}$\ZN$} & \multicolumn{1}{c|}{\cellcolor[HTML]{FFFFC7}$\DN$} & \cellcolor[HTML]{FFFFC7}$\AN$ & \multicolumn{1}{c|}{\cellcolor[HTML]{FFFFC7}$\ZN$} & \cellcolor[HTML]{FFFFC7}$\DN$ \\ \hline \hline
\cellcolor[HTML]{ECF4FF}N1(5) & \multicolumn{1}{c|}{165.35} & \multicolumn{1}{c|}{4.59} & 0.36 & \multicolumn{1}{c|}{0.65} & 0.79 \\
\cellcolor[HTML]{ECF4FF}N1B(6) & \multicolumn{1}{c|}{dnf} & \multicolumn{1}{c|}{328.30} & 15.08 & \multicolumn{1}{c|}{dnf} & 0.89 \\ \hline
\cellcolor[HTML]{ECF4FF}BT10(2) & \multicolumn{1}{c|}{-} & \multicolumn{1}{c|}{1.20} & 0.30 & \multicolumn{1}{c|}{-} & 0.92 \\
\cellcolor[HTML]{ECF4FF}BT5(3) & \multicolumn{1}{c|}{0.54} & \multicolumn{1}{c|}{0.14} & 0.06 & \multicolumn{1}{c|}{0.38} & 0.51 \\
\cellcolor[HTML]{ECF4FF}BT1(4) & \multicolumn{1}{c|}{146.69} & \multicolumn{1}{c|}{50.81} & 3.51 & \multicolumn{1}{c|}{0.95} & 1.03 \\
\hline
\cellcolor[HTML]{ECF4FF}K1(3) & \multicolumn{1}{c|}{32.31} & \multicolumn{1}{c|}{4.97} & 1.37 & \multicolumn{1}{c|}{0.82} & 0.89 \\ \hline
\cellcolor[HTML]{ECF4FF}UM4(2) & \multicolumn{1}{c|}{-} & \multicolumn{1}{c|}{8.47} & 2.43 & \multicolumn{1}{c|}{-} & 0.90 \\
\cellcolor[HTML]{ECF4FF}UM2(3) & \multicolumn{1}{c|}{13.35} & \multicolumn{1}{c|}{1.22} & 0.04 & \multicolumn{1}{c|}{1.04} & 1.52 \\
\hline
\end{tabular}
\end{table}

\begin{table}[]
\caption{Comparison of running time and solution length between $\XA$ (using \loc and \glo) and uniform random sampling. For \rnd we report the average values in terms of running and solution length (the latter is given as normalized value with  respect to the solution length with $\XA$).}
\label{tbl:lattice_vs_random}
\centering
\begin{tabular}{|c||ccc|c|c|}
\hline
 & \multicolumn{3}{c|}{\cellcolor[HTML]{EFEFEF} Time (s)} & \cellcolor[HTML]{EFEFEF}Length (r) & \cellcolor[HTML]{EFEFEF}Success (\%) \\ \cline{2-6} 
\multirow{-2}{*}{\begin{tabular}[c]{@{}c@{}}Scenario\\ (robot \#)\end{tabular}} & \multicolumn{1}{c|}{\cellcolor[HTML]{FFFFC7}\begin{tabular}[c]{@{}c@{}}$\AN$\\ \loc \end{tabular}} & \multicolumn{1}{c|}{\cellcolor[HTML]{FFFFC7}\begin{tabular}[c]{@{}c@{}}$\AN$\\ \glo \end{tabular}} & \cellcolor[HTML]{FFFFC7}\begin{tabular}[c]{@{}c@{}}\rnd\\ \glo\end{tabular} & \cellcolor[HTML]{FFFFC7}\begin{tabular}[c]{@{}c@{}}\rnd\\ \glo\end{tabular} & \cellcolor[HTML]{FFFFC7}\begin{tabular}[c]{@{}c@{}}\rnd\\ \glo\end{tabular} \\ \hline\cellcolor[HTML]{ECF4FF}N1(5) & \multicolumn{1}{c|}{0.36} & \multicolumn{1}{c|}{3.05} & 4.16 & 1.48 & 80 \\
\cellcolor[HTML]{ECF4FF}N1(5) & \multicolumn{1}{c|}{0.36} & \multicolumn{1}{c|}{3.05} & 4.16 & 1.48 & 80 \\
\cellcolor[HTML]{ECF4FF}N2(5) & \multicolumn{1}{c|}{0.41} & \multicolumn{1}{c|}{2.67} & 2.74 & 2.43 & 65 \\
\cellcolor[HTML]{ECF4FF}N3(5) & \multicolumn{1}{c|}{0.59} & \multicolumn{1}{c|}{3.83} & 5.44 & 2.02 & 85 \\ \hline
\cellcolor[HTML]{ECF4FF}BT3(3) & \multicolumn{1}{c|}{5.38} & \multicolumn{1}{c|}{14.15} & 62.22 & 1.05 & 100 \\
\cellcolor[HTML]{ECF4FF}BT8(3) & \multicolumn{1}{c|}{12.17} & \multicolumn{1}{c|}{19.31} & 169.32 & 1.00 & 100 \\
\cellcolor[HTML]{ECF4FF}BT8B(3) & \multicolumn{1}{c|}{3.17} & \multicolumn{1}{c|}{3.60} & 41.63 & 1.16 & 100 \\ \hline
\cellcolor[HTML]{ECF4FF}UM4(2) & \multicolumn{1}{c|}{2.43} & \multicolumn{1}{c|}{2.93} & 12.71 & 0.96 & 70 \\
\cellcolor[HTML]{ECF4FF}UM1(3) & \multicolumn{1}{c|}{6.68} & \multicolumn{1}{c|}{58.62} & 49.35 & 0.98 & 100 \\
\cellcolor[HTML]{ECF4FF}UM2(3) & \multicolumn{1}{c|}{0.04} & \multicolumn{1}{c|}{2.94} & 4.49 & 1.95 & 75 \\ \hline
\cellcolor[HTML]{ECF4FF}ZZB1(2) & \multicolumn{1}{c|}{0.44} & \multicolumn{1}{c|}{0.49} & 10.44 & 0.89 & 100 \\
\cellcolor[HTML]{ECF4FF}ZZB2(2) & \multicolumn{1}{c|}{0.71} & \multicolumn{1}{c|}{2.51} & 272.70 & 0.89 & 100 \\
\cellcolor[HTML]{ECF4FF}ZZB3(2) & \multicolumn{1}{c|}{0.47} & \multicolumn{1}{c|}{7.69} & 341.84 & 0.88 & 100 \\ \hline
\end{tabular}
\end{table}

\subsection{Comparison with randomized sampling}
We compare the performance of iA* when using $\XA$-sampling and uniform random sampling (denoted by \rnd). We show that the advantage of $\XA$ stems not only from algorithmic speedups due to the regular lattice structure (as in \loc) but also from structural properties leading to a more efficient coverage of space. Hence, we run both \glo and \loc using $\XA$, where \rnd is tested only with \glo. 

For a given scenario, we fix the parameters $\delta$ and $\eps$ used to derive $\XA$. To derive a set of random samples $\XR$, we compute the number of $\XA$ samples within the scenario (while ignoring collision with obstacles and between robots) and produce the same number of points via uniform random sampling. Importantly, when running \glo with $\XR$ we specify the standard connection radius  $r_\text{rnd}(n)=\psi \left(\frac{\log n}{n}\right)^{1/d}$, where $n:=|\XR|$ and $\psi$ is a constant, such that asymptotic optimality is guaranteed~\cite{karaman2011sampling}. The results are reported in  Table~\ref{tbl:lattice_vs_random}, where \rnd is averaged over 20 repetitions. (Results for additional scenarios are provided in\conditionaltext{~\Cref{tbl:lattice_vs_random:app} in the appendix}{ the supplementary material}.)

Note that \rnd achieves a perfect success rate only in some of the scenarios. Furthermore, its running time is typically at least $5\times$ slower than \loc, and in some scenarios, up to three orders of magnitude slower. This time gap can be partially attributed to \rnd constructing and maintaining a nearest-neighbor data structure. However, notice that in some scenarios, $\AN$-\glo significantly outperforms \rnd. One explanation of that is $\AN$ requires a smaller connection radius than \rnd, which leads to a lower sample complexity and collision check complexity. Another reason is that \rnd, due to poor coverage of space, is forced to make detours and so considers additional vertices and collision checks. We provide further evidence for those points in\conditionaltext{~\Cref{tbl:lattice_vs_random:app} in the appendix}{ the supplementary material}. We also report results for \rnd where the connection radius $r_\text{rnd}(n)$ is substituted with $r^*$, which further reduces success rates and emphasizes the efficiency of space coverage using $\XA$.



Overall, we conclude that locality can substantially speed up performance from an algorithmic perspective, while the structural properties of $\XA$ also improve performance in terms of success rate, running time, and solution quality. 

\subsection{Tests on a Panda arm}
Here, we evaluate the performance of iA*, which is implemented in the VAMP environment. Specifically, we compare iA*-\loc, using the sample set $\XA$, 

with iA*-\rnd, where the former again proves to be significantly better
(see results in Table~\ref{table:manip_tests}). We set $\eps=10$ and lowered the parameter $\delta$ until iA*-\loc managed to find a solution. 

In the "1x" column, we report the results for \rnd, where the number of samples equals $|\XA|$, as we have previously done. Here, other than the limbo scenario, \rnd achieves a similar running time to \loc. However, other than in two scenarios, \rnd achieves poor success rates. 

We then increase the number of samples used by \rnd by $10$ and $100$ (reported in the columns "10x" and "100x") to find how many samples are necessary so that \rnd achieves near-perfect success rates. Indeed, in the window and sleeve scenarios, the success rate improves with additional samples, but with a significant running-time penalty.  We also attempted this modification in the room scenario, but it was too computationally heavy even for 10x.

\begin{figure*}[htbp]
    \centering
    \subfloat[Window]{              \includegraphics[width=0.3\textwidth,clip]{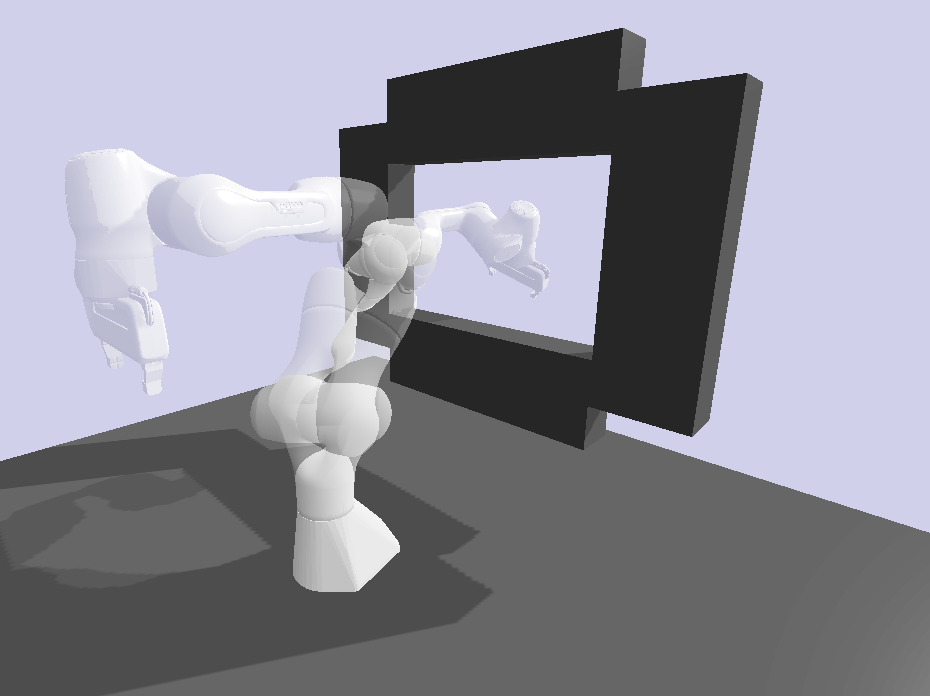}
    \label{fig:manip:window}
    }
    \subfloat[Sleeve]{              \includegraphics[width=0.292\textwidth,clip]{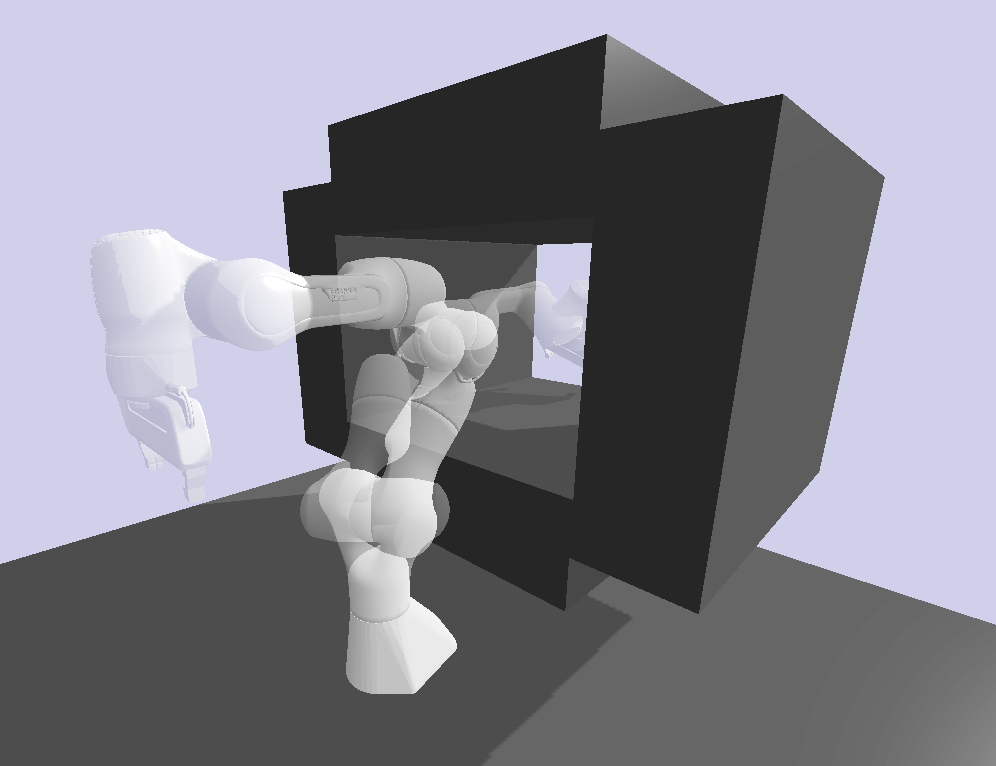}
    \label{fig:manip:sleeve}
    }
    \subfloat[Limbo]{              \includegraphics[width=0.3\textwidth,clip]{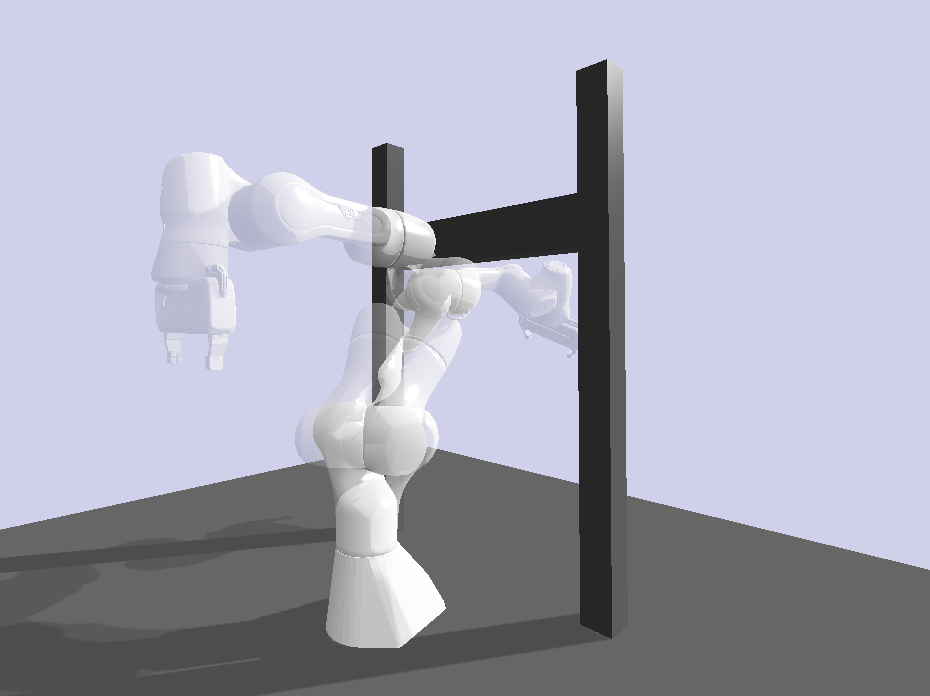}
    \label{fig:manip:limbo}
    }
    
    \vspace{0.5em}  

    \subfloat[Sandwich]{              \includegraphics[width=0.323\textwidth,clip]{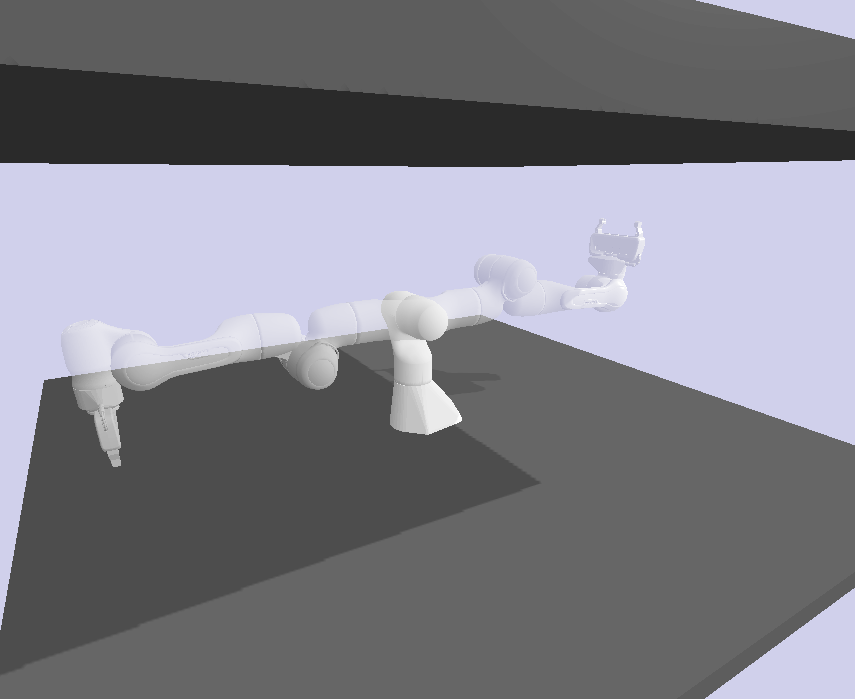}
    \label{fig:manip:sandwich}
    }
    \subfloat[Room]{              \includegraphics[width=0.5\textwidth,clip]{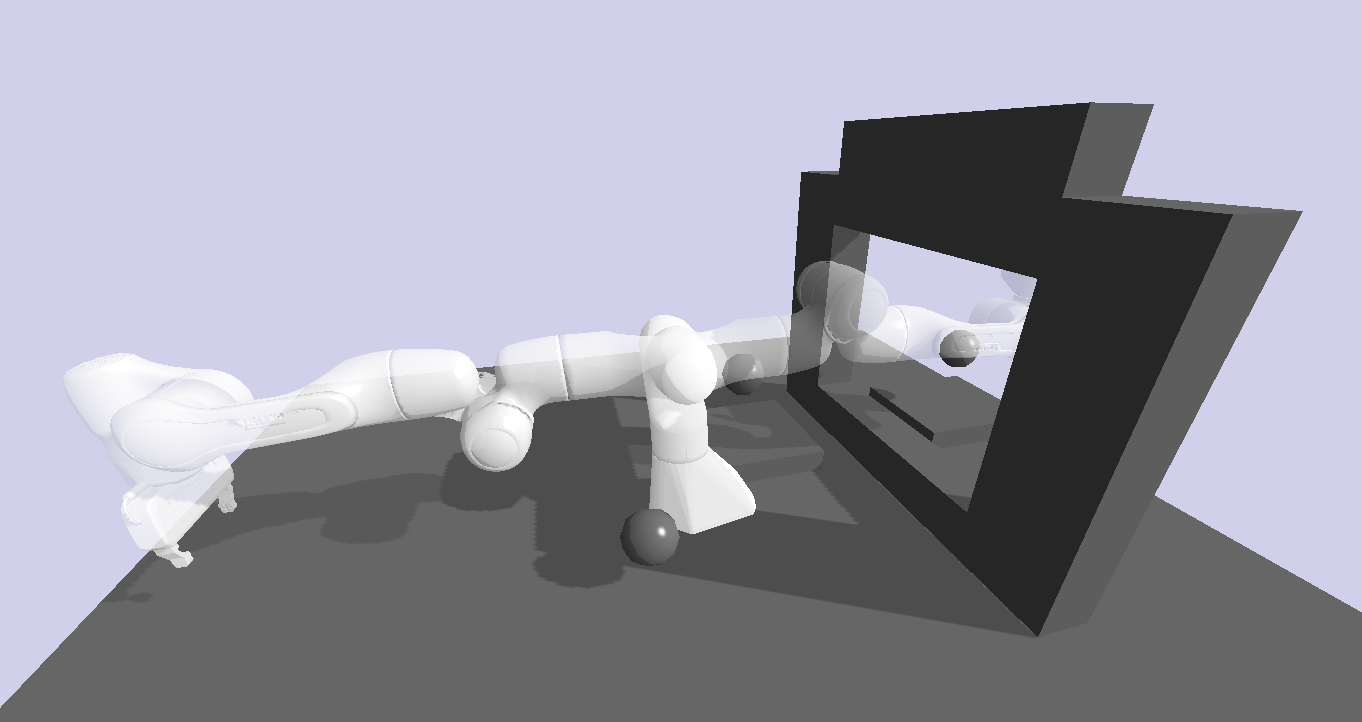}
    \label{fig:manip:room}
    }
    \caption{Manipulator tests in the VAMP environment, depicting the start and target configurations. Window: The arm moves through a window frame.
    Sleeve: The arm moves through an elongated window frame.
    Limbo: The arm ducks under a pole. 
    Sandwich: The arm is being constrained to move in a low-ceiling room.
    Room: The arm moves through a low window frame, while dodging several obstacles on the floor.}
    \label{fig:manip_tests}
\end{figure*}

\begin{table*}[]
\centering
\caption{Manipulator test results, where \rnd results are averaged over 100 runs. Running times are reported in milliseconds.}
\label{table:manip_tests}
\begin{tabular}{ccccc|ccc|ccc|}
\cline{3-11}
\cellcolor[HTML]{FFFFFF} & \multicolumn{1}{c|}{\cellcolor[HTML]{FFFFFF}} & \multicolumn{3}{c|}{\cellcolor[HTML]{ECF4FF}1x samples} & \multicolumn{3}{c|}{\cellcolor[HTML]{ECF4FF}10x samples} & \multicolumn{3}{c|}{\cellcolor[HTML]{ECF4FF}100x samples} \\ \hline
\multicolumn{1}{|c|}{Scenario} & \multicolumn{2}{c}{\cellcolor[HTML]{EFEFEF}Total (ms)} & \cellcolor[HTML]{EFEFEF}Length (r) & \cellcolor[HTML]{EFEFEF}Success (\%) & \cellcolor[HTML]{EFEFEF}Total (ms) & \cellcolor[HTML]{EFEFEF}Length (r) & \cellcolor[HTML]{EFEFEF}Success (\%) & \cellcolor[HTML]{EFEFEF}Total (ms) & \cellcolor[HTML]{EFEFEF}Length (r) & \cellcolor[HTML]{EFEFEF}Success (\%) \\ \cline{2-11} 
\multicolumn{1}{|c|}{} & \multicolumn{1}{c|}{\cellcolor[HTML]{FFFFC7}\begin{tabular}[c]{@{}c@{}}$\AN$ \\ \loc\end{tabular}} & \cellcolor[HTML]{FFFFC7}\begin{tabular}[c]{@{}c@{}}\rnd\\ \glo\end{tabular} & \cellcolor[HTML]{FFFFC7}\begin{tabular}[c]{@{}c@{}}\rnd\\ \glo\end{tabular} & \cellcolor[HTML]{FFFFC7}\begin{tabular}[c]{@{}c@{}}\rnd\\ \glo\end{tabular} & \cellcolor[HTML]{FFFFC7}\begin{tabular}[c]{@{}c@{}}\rnd\\ \glo\end{tabular} & \cellcolor[HTML]{FFFFC7}\begin{tabular}[c]{@{}c@{}}\rnd\\ \glo\end{tabular} & \cellcolor[HTML]{FFFFC7}\begin{tabular}[c]{@{}c@{}}\rnd\\ \glo\end{tabular} & \cellcolor[HTML]{FFFFC7}\begin{tabular}[c]{@{}c@{}}\rnd\\ \glo\end{tabular} & \cellcolor[HTML]{FFFFC7}\begin{tabular}[c]{@{}c@{}}\rnd\\ \glo\end{tabular} & \cellcolor[HTML]{FFFFC7}\begin{tabular}[c]{@{}c@{}}\rnd\\ \glo\end{tabular} \\ \hline
\multicolumn{1}{|c|}{\cellcolor[HTML]{ECF4FF}window} & \multicolumn{1}{c|}{3.6} & 2.9 & 1.3 & 13 & 43.31 & 1.13 & 34 & 526.2 & 1.06 & 94 \\
\multicolumn{1}{|c|}{\cellcolor[HTML]{ECF4FF}sleeve} & \multicolumn{1}{c|}{160.2} & 108.9 & 0.77 & 26 & 1510.8 & 0.69 & 51 & 18800.2 & 0.66 & 100 \\
\multicolumn{1}{|c|}{\cellcolor[HTML]{ECF4FF}limbo} & \multicolumn{1}{c|}{82.2} & 675.7 & 0.7 & 96 & - & - & - & - & - & - \\
\multicolumn{1}{|c|}{\cellcolor[HTML]{ECF4FF}sandwich} & \multicolumn{1}{c|}{9.7} & 5.1 & 1.6 & 100 & - & - & - & - & - & - \\
\multicolumn{1}{|c|}{\cellcolor[HTML]{ECF4FF}room} & \multicolumn{1}{c|}{9718} & 14823 & 1.0 & 9 & - & - & - & - & - & - \\ \hline
\end{tabular}
\end{table*}

\section{Limitations and future work}\label{sec:future}
We leveraged foundational results in lattice theory to develop a theoretical framework for generating efficient sample sets for motion planning that endow their planners with finite-time guarantees, which are vastly superior to previous asymptotic properties. We demonstrated the practical potential of the framework in challenging motion-planning scenarios wherein our $A_d^*$-based sampling procedures lead to substantial speeds over previous methods (deterministic and random).
Below, we discuss several limitations of our work, which motivate further research directions.

\niceparagraph{Multi-resolution search.} The choice of the sampling parameters $\delta$ and $\epsilon$ can significantly impact the sample and collision-check complexity of the sampling distribution, and hence the planner's performance. In this work, we selected those values according to a rule of thumb, which might not be generally applicable to broader problem settings, or lead to feasible solutions.  We plan on exploring algorithmic extension, which would allow the automatic selection of those parameters during runtime. One promising direction is exploiting multi-resolution search~\cite{saxena2022amra,FuSSA23}, which allows for adjusting sampling densities according to the space's structure. This has the potential of not only simplifying the usage of our sample sets but also significantly improving the planner's performance.  

\niceparagraph{Non-Euclidean systems.} Our current work focuses on geometric systems (whose state space is Euclidean)  and path-length cost functions, which limits its applicability. In the future, we plan to extend our work to more general settings, which entails at least two challenges. First, the adaptation of our sampling distributions to non-Euclidean spaces is non-trivial due to the heterogeneity of the space coordinates, e.g., position and rotation components in SE(3)~\cite{yershova2004deterministic}, or state derivatives in dynamical systems~\cite{janson2018deterministic}. Second, those settings call for more general cost functions, which requires revising the concept of \decomps and its relation to geometric coverage (particularly, Lemma~\ref{lem:cover}). We plan to address those challenges incrementally. As a first step, we plan on generalizing our theory to linear systems with quadratic costs, where we believe that the structure of LQR controllers~\cite{liberzon2011calculus} can be useful. More general approaches could be obtained from exploiting the locally-linear structure of nonlinear systems using differential geometry~\cite{BulloLewis04}.

\niceparagraph{Incremental sampling.} Uniform random sampling is naturally amenable for densification, i.e., the incremental addition of new sampling points, making them applicable for anytime planner like RRT or RRT*~\cite{LaVKuf01,karaman2011sampling}. In contrast, lattice-based sampling requires the introduction of large batches, which currently limits their applicability. One way to bridge this gap could be generalizing Halton sequences~\cite{kuipers2012uniform}, a popular method for deterministic incremental low-discrepancy sampling, to the $A^*_d$ lattice. Halton sequences ensue from transforming points from the standard grid lattice $\dZ^d$, and it would be interesting to understand the structure emerging from feeding those points through the generator matrix of the lattice $A^*_d$. Another interesting direction is incrementally introducing the elements of lattice-based samples according to a random permutation.

\section*{Acknowledgments}
This research was partly supported by Israel's Ministry of Innovation, Science and Technology,  Ministry of Transport,  National Road Safety Authority, and Netivei Israel, Grant No. 0006730.

We thank Rom Pinchasi, Roy Meshulam, Dan Halperin, and Oren Salzman, and Steve LaValle for fruitful discussions. We also thank Ido Jacobi, Roy Steinberg and Yaniv Hasidoff for comments on the manuscript. 

\ifincludeappendix
\appendix
  \subsection{Additional details for Theorem~1}\label{app:decomp_lattice_proof}
We provide details that were omitted in the proof of Theorem~1. First, we derive the matrix $P$.
Given the planes $H$ and  $H_0:=\{x_{d+1}=0\}$, we wish  to find a plane $H_{ref}$ that is half-way (angle-wise) between $H$ and $H_0$. This would allow to reflect points in $H$ onto $H_0$ through $H_{ref}$ where the reflection is achieved using the Householder matrix $P:=I-2\hat{n}_{ref}\hat{n}^t_{ref}$, where $\hat{n}_{ref}\in \dR^{d+1}$ is the normal of $H_{ref}$~\cite{householder1958unitary}. That is, we reflect a lattice point $p\in \dR^{d+1}$ by computing the value \({p_{\text{reflected}}=P\cdot p}\).

Next, we show that the normal 
\begin{equation*}
    \hat{n}_{ref}:=\frac{1}{\sqrt{2-\frac{2}{\sqrt{d+1}}}}\cdot \left(-\tfrac{1}{\sqrt{d+1}},\dots,-\tfrac{1}{\sqrt{d+1}},1-\tfrac{1}{\sqrt{d+1}}\right)
\end{equation*}
satisfies those requirements.\footnote{We obtained the expression for $\hat{n}_{ref}$ by first considering $d=2$, where the task is more tangible, and then generalizing to higher dimensions.}  Consider the Householder matrix 
\begin{align}
\label{eq:reflection}
 P&= 
 I-2\hat{n}_{ref}\hat{n}^t_{ref}\nonumber\\
 &=\begin{pNiceArray}{cw{c}{1cm}c|c}[margin]
            \Block{3-3}<\Large>{I_d - \frac{1}{D-\sqrt{D}}\mathds{1}} 
            & & & \dfrac{1}{\sqrt{D}} \\
            & & & \Vdots \\
            & & & \dfrac{1}{\sqrt{D}} \\
            \hline
            \dfrac{1}{\sqrt{D}} & \dots& \dfrac{1}{\sqrt{D}} & \dfrac{1}{\sqrt{D}}
    \end{pNiceArray},
\end{align}
where $D:=d+1$, $I_d$ is an $d\times d$ identity matrix, and $\mathds{1}$ is the $d\times d$ matrix with $1$s in all its entries. 

Consider a point $p\in A^*_d$. Next, we show that it is reflected onto the plane  $H_0$, i.e., for $v=P\cdot p$, we get $v_{d+1}=0$. To do that, we move to the basis of the integer vector space, and show that for all $1\leq i\leq d$, taking the base element $e_i=(0,\dots,1,\dots,0)$, the $(d+1)$th element of $v=PG^t\cdot e_i$ (i.e., using the generator and then the reflector) is zero. First, for all $i<d$ it holds that 
    \begin{align*}
        PG^t\cdot e_i=P\cdot
        \begin{pmatrix}
        1 &  0&  \dots&  0& -1& 0 &\dots& 0
        \end{pmatrix}^t.
    \end{align*}
Now, considering that the elements of the final row of $P$ are all equal to $1/\sqrt{D}$, we obtain a zero in the $(d+1)$th dimension. It remains to calculate the expression resulting from multiplying with $e_d$:
        \begin{align*}
        PG^t\cdot e_d=P\cdot
        \begin{pmatrix}
        -\frac{D-1}{D} &  \frac{1}{D}&  \dots&  \frac{1}{D}
        \end{pmatrix}^t.
    \end{align*}
    Looking specifically at the last element, we see that it is equal to     \begin{align*}
        \frac{1}{\sqrt{D}}\cdot\frac{1-D}{D} + (D-1)\frac{1}{D}\frac{1}{\sqrt{D}}=\frac{1-D+D-1}{D\sqrt{D}}=0.
    \end{align*}

    That is, by applying the transformation $P$ on the lattice points, we reflect them onto the $x_{d+1}=0$ plane. It remains to get rid of the $(d+1)$th dimension. This is accomplished by the mapping
\begin{align*}
        E=
        \begin{pmatrix}
            1 & 0 & \dots & 0 & 0 \\
            0 & 1 & \dots & 0 & 0 \\
            \vdots & \vdots & \ddots & \vdots & 0 \\
            0 & 0 & \dots & 1 & 0
        \end{pmatrix}_{d\times(d+1)}.
    \end{align*}
    
It remains to compute the  explicit embedding $T(g):=EPG^t(g)$, for $g\in \dZ^d$. We first calculate 
    \begin{align*}
        \left(EP\right)^t=
        \begin{pNiceArray}{cw{c}{1cm}c}[margin]
            \Block{3-3}<\Large>{I_d - \frac{1}{D-\sqrt{D}}\mathds{1}} 
            & &  \\
            & &  \\
            & &  \\
            \hline
            \dfrac{1}{\sqrt{D}} & \dots & \dfrac{1}{\sqrt{D}}
        \end{pNiceArray}_{d\times(d+1)}.
    \end{align*}
 Next, it can be shown that
    \begin{align}
        T^t&=G\left(EP\right)^t\nonumber\\
        &=\begin{pmatrix}
            1 & -1 &  0  & \dots & 0 \\
            1 & 0  &  -1 & \dots & 0 \\
            \vdots & \vdots  &  \vdots  & \ddots & \vdots \\
            1 & 0  &  0  & \dots & -1 \\
            \frac{1}{D - \sqrt{D}} - 1 & \frac{1}{D - \sqrt{D}} & \frac{1}{D - \sqrt{D}} & \dots & \frac{1}{D - \sqrt{D}}
        \end{pmatrix}_{d\times(d+1)}\!\!\!\!\!\!.
    \end{align}


\subsection{Additional details for Theorem~3}\label{app:CC}
We provide details omitted from the main body of the text. 
We start with a simplified derivation of a single annulus, which would inform the more advanced construction. Fix ${0<r_1<
  r^*}$ forced it to a single line, and define $\btheta_{r'} := \frac{r'}{{\beta^*}}f_\Lambda$, and observe that 
\begin{align}
  CC_\X&\leq  r^*\cdot
\left|\X\cap (\B_{r^*}\setminus \B_{r_1})\right| + r_1\cdot
\left|\X\cap \B_{r_1}\right|\nonumber \\ & = r^*\left(|\X\cap \B_{r^*}|-|\X \cap \B_{r_1}|\right) + r_1 \left|\X\cap \B_{r_1}\right| \nonumber\\
& = r^*|\X\cap \B_{r^*}|+ (r_1-r^*) |\X\cap \B_{r_1}| \nonumber \\  
  & = r^*\frac{\partial(B_1)}{\sqrt{\det(\Lambda)}}\btheta^d_{r^*} +r^* P_d(\btheta_{r^*}) + (r_1-r^*) \frac{\partial(B_1)}{\sqrt{\det(\Lambda)}}\btheta^d_{r_1}\nonumber\\& + (r_1-r^*) P_d(\btheta_{r_1}) \nonumber
\\
& = \frac{\partial(B_1)}{\sqrt{\det(\Lambda)}}\theta^d\left({r^*}^{d+1}+{r_1}^{d+1}-r^*{r_1}^{d}\right)\nonumber\\&+ r P_d(\btheta_{r^*}) + (r_1-r^*) P_d(\btheta_{r_1})\nonumber \\ & = \frac{\partial(B_1)}{\sqrt{\det(\Lambda)}}\theta^d\left({r^*}^{d+1}+{r_1}^{d+1}-r^*{r_1}^{d}\right)+ r^* P_d(\btheta_{r^*}), \label{eq:CC1}
\end{align}
where the sample complexity bound in Equation~(5) is used. For simplicity, we bound throughout the error term with $r^* P_d(\btheta_{r^*})$.
Next, we optimize the value $r_1$ to minimize the expression in Equation~\eqref{eq:CC1}.

Consider the function $f(r_1)={r^*}^{d+1}-{r^*} r^d_1 + {r^*}^{d+1}_1$. We look for the minimum of $f(r_1)$ by requiring that
\begin{align*}
            f'(r_1)=-{r^*} dr_1^{d-1}+(d+1)r_1^d=0,
\end{align*}
which yields the value $r'_1:=\frac{d}{d+1}{r^*}$. This value is  a minimum since
\begin{align*}
 f^{(2)}(r_1)|_{r'_1}=&\left(-{r^*}(d-1)r_1^{d-2}+d(d+1)r_1^{d-1}\right)|_{r_1'}\\
        =&{r^*}^{d-1}\left(\frac{d^d}{(d+1)^{d-2}}-\frac{d^{d-2}(d-1)}{(d+1)^{d-2}}\right)\\
        =&{r^*}^{d-1}d^{d-2}\frac{d^2-d+1}{(d+1)^{d-2}},
    \end{align*}
    and we know that $d^2-d+1>0$ for all $d\geq 2$.
    
Now, we apply the above line of reasoning in a recursive manner by considering a sequence of $k+1\geq 2$ radii ${0<r_k<\ldots<r_0={r^*}}$ where $r_i:=\td^i r^*$, where $\td:=\frac{d}{d+1}$. This leads to the bound
\begin{align}
\label{eq:cc_eval_app}
CC_\X&\leq \sum_{i=0}^{k-1}r_i |\X\cap (\B_{r_i}\setminus \B_{r_{i+1}})| + r_k|\X\cap \B_{r_k}|\nonumber\\
  &= \frac{\partial(B_1)}{\sqrt{\det(\Lambda)}} \left(\underbrace{{r^*} \btheta^d_{r^*} + \sum_{i=1}^k(r_i-r_{i-1}) \btheta^d_{r_i}}_{:=\gamma}\right) + {r^*} P_d(\btheta_{r^*}).
\end{align}

We show that 
\[\gamma:=r \btheta^d_{r^*} + \sum_{i=1}^k(r_i-r_{i-1}) \btheta^d_{r_i}= {r^*} \btheta^d_{r^*} \left(1 - \frac{\xi^{d+2} - \xi}{ d\xi - (d+1)}\right),\]
where $r_i=\td^i {r^*},\td:=\frac{d}{d+1}, \btheta_{r_i}= r_i\frac{\btheta_{r^*}}{r^*}, k=d,$ and $\xi:=\td^d=\left(\frac{d}{d+1}\right)^d$. In particular,
\begin{align}
  \gamma &={r^*} \btheta^d_{r^*} + \sum_{i=1}^k(r_i-r_{i-1}) r_i^d\frac{\btheta^d_{r^*}}{{r^*}^d} \nonumber\\
  & = {r^*} \btheta^d_{r^*} + \sum_{i=1}^k{r^*}\td^{i-1}(\td-1) \td^{di} {r^*}^d\frac{\btheta^d_{r^*}}{{r^*}^d} \nonumber\\ 
  &=  {r^*} \btheta^d_{r^*} + \sum_{i=1}^k{r^*} (\td-1) \td^{di+ i -1} \btheta^d_{r^*} \nonumber\\ 
  &=   {r^*} \btheta^d_{r^*} \left(1 + \sum_{i=1}^k (\td-1) \td^{di+ i -1} \right)\nonumber\\
  &= {r^*} \btheta^d_{r^*} \left(1 + \frac{\td-1}{\td}\sum_{i=1}^k \td^{(d+1)i} \right)\nonumber\\
  &= {r^*} \btheta^d_{r^*} \left(1 + \frac{\td-1}{\td}\frac{\left(\td^{d+1}\right)^{k+1} - \td^{d+1}}{\td^{d+1} - 1} \right)\nonumber\\
  &= {r^*} \btheta^d_{r^*} \left(1 + \td^d(\td-1)\frac{\left(\td^{d+1}\right)^k - 1}{\td^{d+1} - 1} \right).\nonumber\\
\end{align}

Taking $k=d$ results in $r_k=(\frac{d}{d+1})^d {r^*}\approx\frac{1}{e}{r^*}$. 
To use the sample set analysis, we need a large enough $r$ value, so assuming the original $r$ is large enough, we can deduce safely that $\frac{r}{e}$ is also large enough. 
Notice also that $\td - 1 = \frac{-1}{d+1}$, and thus $(d+1)\td=d$, so returning to our expression, and substituting $\xi:=\td^d=\left(\frac{d}{d+1}\right)^d$, we obtain 
\begin{align*}
    \gamma&={r^*} \btheta^d_{r^*} \left(1 - \frac{\td^d\left(\td^{d(d+1)} - 1\right)}{(d+1)(\td^{d+1} - 1)}\right)\\
    &= {r^*} \btheta^d_{r^*} \left(1 - \frac{\xi\left(\xi^{d+1} - 1\right)}{(d+1)(\td \xi - 1)}\right) 
    \\&= {r^*} \btheta^d_{r^*} \left(1 - \frac{\xi^{d+2} - \xi}{ d\xi - (d+1)}\right)
    :={r^*} \btheta^d_{r^*}\zeta.
\end{align*}

We finish this section with a plot of the value $\gamma$ in Figure~\ref{fig:annuli_bound:app}.

\subsection{Additional experimental results}
Additional scenarios, which were omitted from the main paper, are given in Figure~\ref{fig:scenarios:app}. Extended results comparing lattice-based samples using the \loc algorithm are provided in Table~\ref{tbl:lattice_comparison:app}.

\begin{figure}[H]
\centering  
\includegraphics[width=0.9\columnwidth]{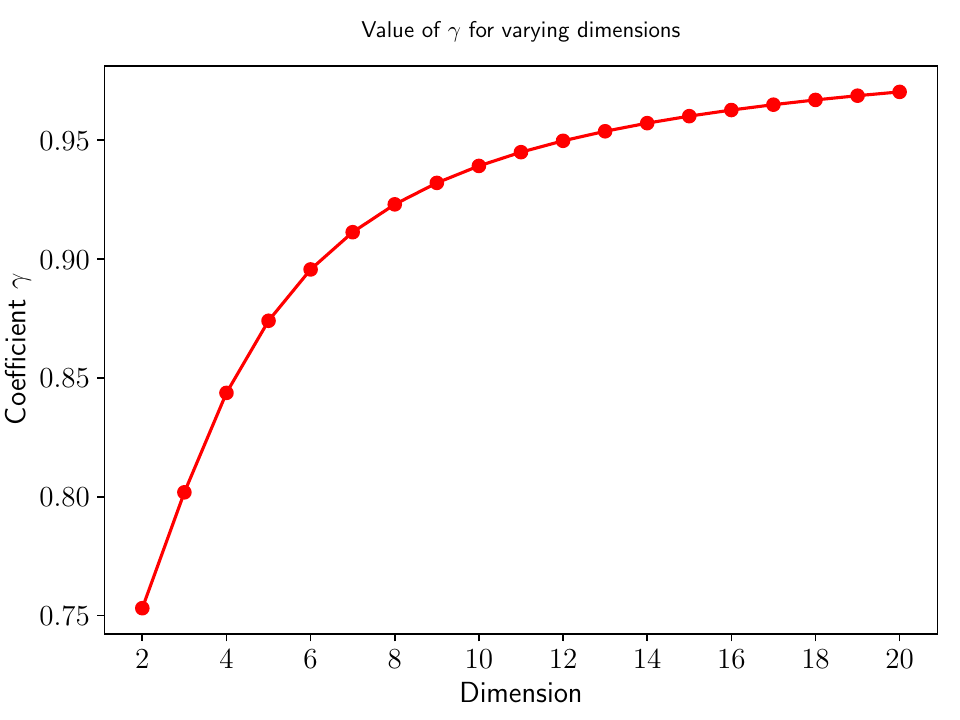}
\caption{Plot of the improvement factor $\gamma$.}
\label{fig:annuli_bound:app}
\end{figure}

\begin{figure*}[thb]
  \centering
\subfloat[Zigzag-bypass (short)]{\includegraphics[width=1.15\columnwidth,clip]{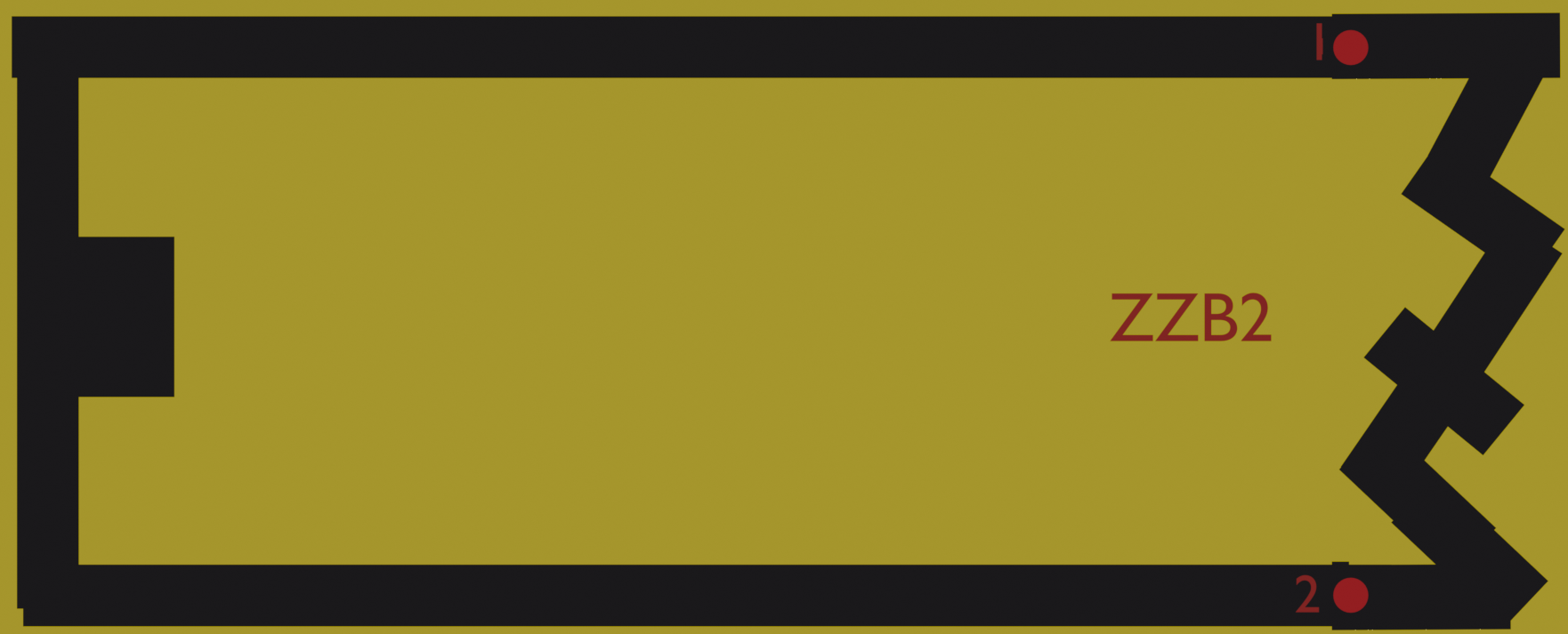}
    }
\subfloat[Narrow (more scenarios)]{\includegraphics[width=0.465\columnwidth,clip]{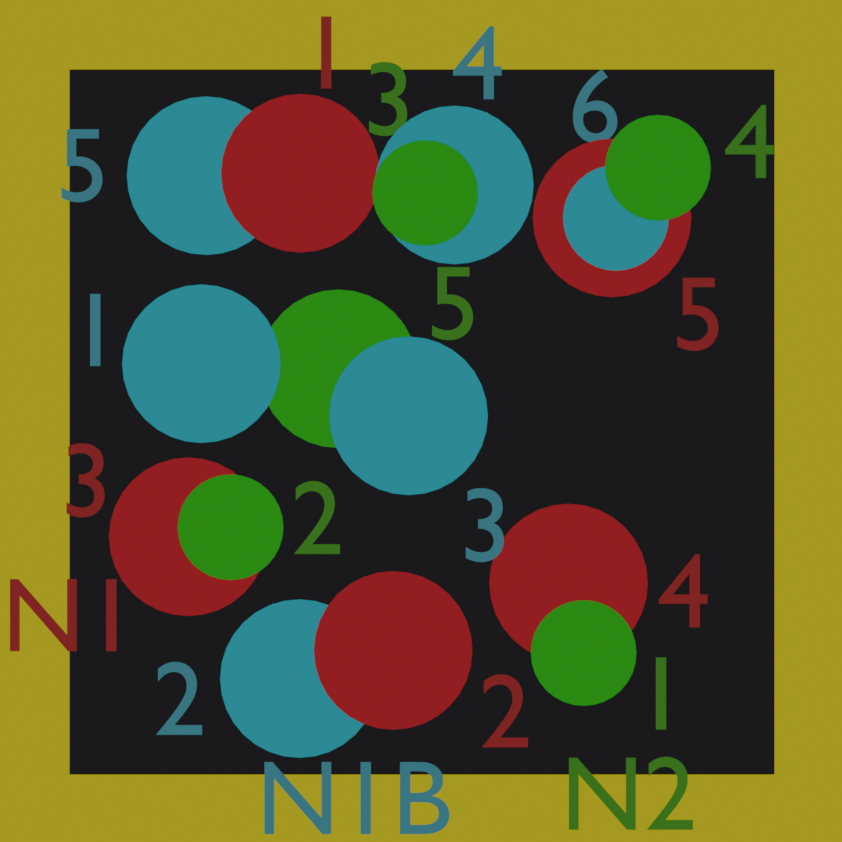}
    }
  \caption{Additional scenarios used in the experiments. The scenario ZZB3, which is not illustrated here, is similar to ZZB2, only that the horizontal hallways are twice as long.}
  \label{fig:scenarios:app}
\end{figure*}

\begin{figure*}[tbh]
  \centering
\subfloat{\includegraphics[width=\columnwidth,clip]{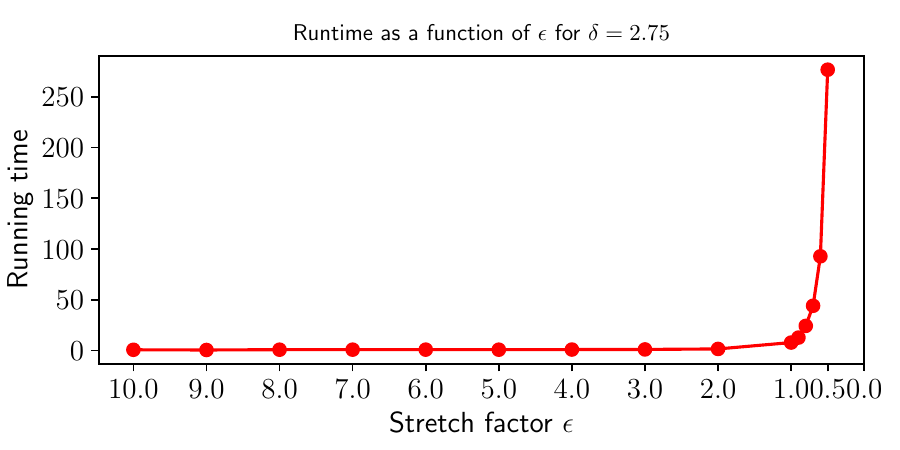}
    }
    \subfloat{\includegraphics[width=\columnwidth,clip]{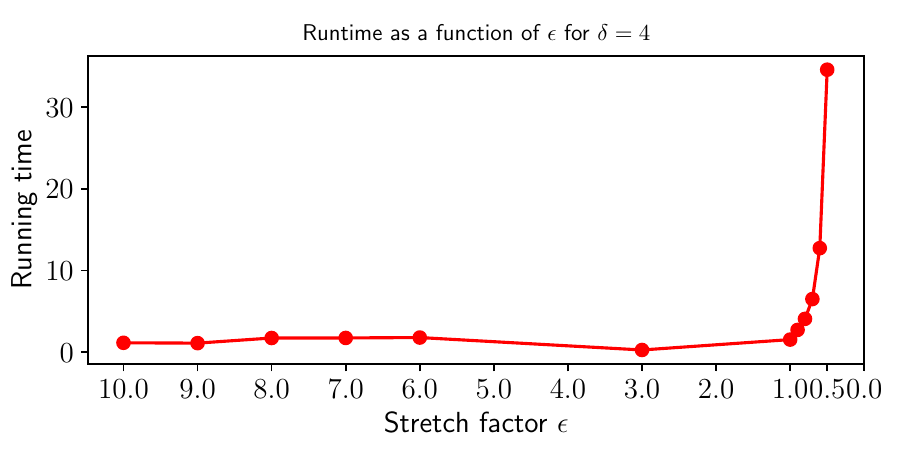}
    }
    \newline
\subfloat{\includegraphics[width=\columnwidth,clip]{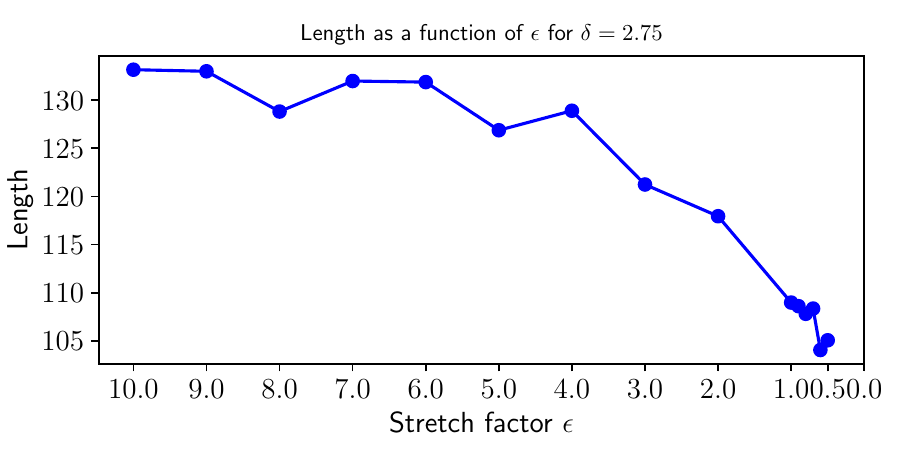}
    }
\subfloat{\includegraphics[width=\columnwidth,clip]{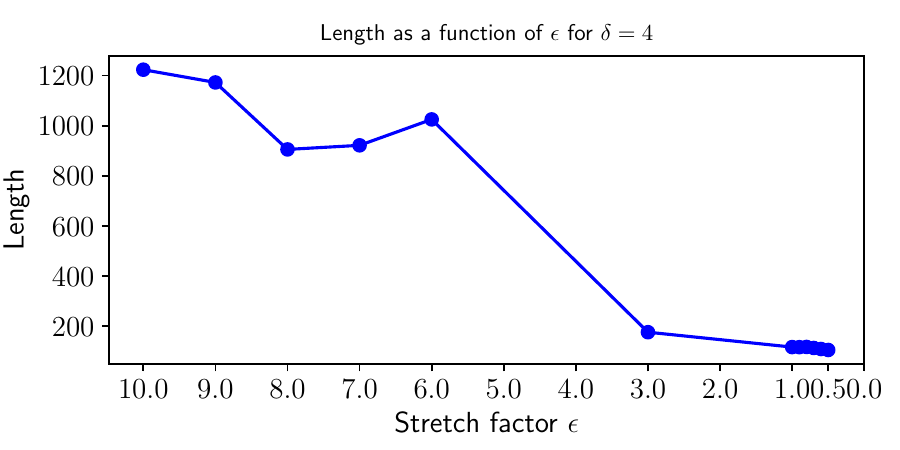}
    }
      \caption{Effect of the parameters $\delta,\epsilon$ on the performance of \loc with $\XA$ for $\delta=2.5$ (left) and $\delta=4$ (right). We report the running time (top) and solution length (bottom). The absence of data points for the parameters $\delta=4, \eps\in \{2,4,5\}$ indicates a solution failure.  
  }
  \label{fig:parameters:app}
\end{figure*}

\begin{table}[tbh]
\caption{Extended comparison of running time and solution length using lattices-based sample sets (where the underlying lattice is denoted in the table) in the iA*-\loc algorithm. Solution length is normalized with respect to the length obtained using $\XA$. }
\centering
\label{tbl:lattice_comparison:app}
\begin{tabular}{|c||ccc|cc|}
\hline
 & \multicolumn{3}{c|}{\cellcolor[HTML]{EFEFEF} Time (s)} & \multicolumn{2}{c|}{\cellcolor[HTML]{EFEFEF} Length (r)} \\ \cline{2-6} 
\multirow{-2}{*}{\begin{tabular}[c]{@{}c@{}}Scenario\\ (robot \#)\end{tabular}} & \multicolumn{1}{c|}{\cellcolor[HTML]{FFFFC7}$\ZN$} & \multicolumn{1}{c|}{\cellcolor[HTML]{FFFFC7}$\DN$} & \cellcolor[HTML]{FFFFC7}$\AN$ & \multicolumn{1}{c|}{\cellcolor[HTML]{FFFFC7}$\ZN$} & \cellcolor[HTML]{FFFFC7}$\DN$ \\ \hline \hline
\cellcolor[HTML]{ECF4FF}N4(2) & \multicolumn{1}{c|}{0.00} & \multicolumn{1}{c|}{0.00} & 0.00 & \multicolumn{1}{c|}{0.62} & 0.74 \\
\cellcolor[HTML]{ECF4FF}N1(5) & \multicolumn{1}{c|}{165.35} & \multicolumn{1}{c|}{4.59} & 0.36 & \multicolumn{1}{c|}{0.65} & 0.79 \\
\cellcolor[HTML]{ECF4FF}N2(5) & \multicolumn{1}{c|}{62.68} & \multicolumn{1}{c|}{1.81} & 0.41 & \multicolumn{1}{c|}{0.85} & 0.95 \\
\cellcolor[HTML]{ECF4FF}N3(5) & \multicolumn{1}{c|}{142.27} & \multicolumn{1}{c|}{2.91} & 0.59 & \multicolumn{1}{c|}{0.65} & 0.87 \\
\cellcolor[HTML]{ECF4FF}N5(5) & \multicolumn{1}{c|}{dnf} & \multicolumn{1}{c|}{4.82} & 3.32 & \multicolumn{1}{c|}{dnf} & 0.82 \\
\cellcolor[HTML]{ECF4FF}N1B(6) & \multicolumn{1}{c|}{dnf} & \multicolumn{1}{c|}{328.30} & 15.08 & \multicolumn{1}{c|}{dnf} & 0.89 \\ \hline
\cellcolor[HTML]{ECF4FF}BT4(2) & \multicolumn{1}{c|}{0.04} & \multicolumn{1}{c|}{0.01} & 0.01 & \multicolumn{1}{c|}{0.69} & 0.85 \\
\cellcolor[HTML]{ECF4FF}BT10(2) & \multicolumn{1}{c|}{-} & \multicolumn{1}{c|}{1.20} & 0.30 & \multicolumn{1}{c|}{-} & 0.92 \\
\cellcolor[HTML]{ECF4FF}BT5(3) & \multicolumn{1}{c|}{0.54} & \multicolumn{1}{c|}{0.14} & 0.06 & \multicolumn{1}{c|}{0.38} & 0.51 \\
\cellcolor[HTML]{ECF4FF}BT1(4) & \multicolumn{1}{c|}{146.69} & \multicolumn{1}{c|}{50.81} & 3.51 & \multicolumn{1}{c|}{0.95} & 1.03 \\
\cellcolor[HTML]{ECF4FF}BT6(4) & \multicolumn{1}{c|}{dnf} & \multicolumn{1}{c|}{153.40} & 12.36 & \multicolumn{1}{c|}{dnf} & 1.04 \\
\cellcolor[HTML]{ECF4FF}BT7(4) & \multicolumn{1}{c|}{240.88} & \multicolumn{1}{c|}{5.38} & 4.36 & \multicolumn{1}{c|}{0.95} & 0.96 \\ \hline
\cellcolor[HTML]{ECF4FF}K1(3) & \multicolumn{1}{c|}{32.31} & \multicolumn{1}{c|}{4.97} & 1.37 & \multicolumn{1}{c|}{0.82} & 0.89 \\ \hline
\cellcolor[HTML]{ECF4FF}UM4(2) & \multicolumn{1}{c|}{-} & \multicolumn{1}{c|}{8.47} & 2.43 & \multicolumn{1}{c|}{-} & 0.90 \\
\cellcolor[HTML]{ECF4FF}UM1(3) & \multicolumn{1}{c|}{482.17} & \multicolumn{1}{c|}{25.15} & 6.68 & \multicolumn{1}{c|}{0.84} & 1.16 \\
\cellcolor[HTML]{ECF4FF}UM2(3) & \multicolumn{1}{c|}{13.35} & \multicolumn{1}{c|}{1.22} & 0.04 & \multicolumn{1}{c|}{1.04} & 1.52 \\
\cellcolor[HTML]{ECF4FF}UM4B3(3) & \multicolumn{1}{c|}{99.35} & \multicolumn{1}{c|}{1.03} & 0.66 & \multicolumn{1}{c|}{1.59} & 0.89 \\
\cellcolor[HTML]{ECF4FF}UM3(4) & \multicolumn{1}{c|}{236.31} & \multicolumn{1}{c|}{223.87} & 64.57 & \multicolumn{1}{c|}{0.63} & 0.97 \\ \hline
\cellcolor[HTML]{ECF4FF}ZZB1(2) & \multicolumn{1}{c|}{1.93} & \multicolumn{1}{c|}{1.01} & 0.44 & \multicolumn{1}{c|}{0.94} & 0.94 \\
\cellcolor[HTML]{ECF4FF}ZZB2(2) & \multicolumn{1}{c|}{2.91} & \multicolumn{1}{c|}{0.93} & 0.71 & \multicolumn{1}{c|}{0.94} & 0.94 \\
\cellcolor[HTML]{ECF4FF}ZZB3(2) & \multicolumn{1}{c|}{2.26} & \multicolumn{1}{c|}{0.84} & 0.47 & \multicolumn{1}{c|}{0.95} & 0.95 \\ \hline\end{tabular}
\end{table}

\begin{table*}[tbh]
\centering
\begin{tabular}{|c|cccl|ccl|cl|cl|}
\hline
 & \multicolumn{4}{c|}{\cellcolor[HTML]{EFEFEF} Total time (s)} & \multicolumn{3}{c|}{\cellcolor[HTML]{EFEFEF} Search time (s)} & \multicolumn{2}{c|}{\cellcolor[HTML]{EFEFEF}Length (r)} & \multicolumn{2}{c|}{\cellcolor[HTML]{EFEFEF}Success (\%)} \\ \cline{2-12} 
\multirow{-2}{*}{\begin{tabular}[c]{@{}c@{}}Scenario\\ (Robot \#)\end{tabular}} & \multicolumn{1}{c|}{\cellcolor[HTML]{FFFFC7}\begin{tabular}[c]{@{}c@{}}$\AN$\\ \loc\end{tabular}} & \multicolumn{1}{c|}{\cellcolor[HTML]{FFFFC7}\begin{tabular}[c]{@{}c@{}}$\AN$\\ \glo\end{tabular}} & \multicolumn{1}{c|}{\cellcolor[HTML]{FFFFC7}\begin{tabular}[c]{@{}c@{}}\rnd\\ \glo\end{tabular}} & \multicolumn{1}{c|}{\cellcolor[HTML]{FFFFC7}\begin{tabular}[c]{@{}c@{}}\rndm\\ \glo\end{tabular}} & \multicolumn{1}{c|}{\cellcolor[HTML]{FFFFC7}\begin{tabular}[c]{@{}c@{}}$\AN$\\ \glo\end{tabular}} & \multicolumn{1}{c|}{\cellcolor[HTML]{FFFFC7}\begin{tabular}[c]{@{}c@{}}\rnd\\ \glo\end{tabular}} & \multicolumn{1}{c|}{\cellcolor[HTML]{FFFFC7}\begin{tabular}[c]{@{}c@{}}\rndm\\ \glo\end{tabular}} & \multicolumn{1}{c|}{\cellcolor[HTML]{FFFFC7}\begin{tabular}[c]{@{}c@{}}\rnd\\ \glo\end{tabular}} & \multicolumn{1}{c|}{\cellcolor[HTML]{FFFFC7}\begin{tabular}[c]{@{}c@{}}\rndm\\ \glo\end{tabular}} & \multicolumn{1}{c|}{\cellcolor[HTML]{FFFFC7}\begin{tabular}[c]{@{}c@{}}\rnd\\ \glo\end{tabular}} & \multicolumn{1}{c|}{\cellcolor[HTML]{FFFFC7}\begin{tabular}[c]{@{}c@{}}\rndm\\ \glo\end{tabular}} \\ \hline
\cellcolor[HTML]{ECF4FF}N1(5) & \multicolumn{1}{c|}{0.36} & \multicolumn{1}{c|}{3.05} & \multicolumn{1}{c|}{4.16} & 3.40 & \multicolumn{1}{c|}{0.84} & \multicolumn{1}{c|}{3.37} & 2.59 & \multicolumn{1}{c|}{1.48} & 1.46 & \multicolumn{1}{c|}{80.00} & 90 \\
\cellcolor[HTML]{ECF4FF}N2(5) & \multicolumn{1}{c|}{0.41} & \multicolumn{1}{c|}{2.67} & \multicolumn{1}{c|}{2.74} & 4.28 & \multicolumn{1}{c|}{0.82} & \multicolumn{1}{c|}{2.11} & 3.62 & \multicolumn{1}{c|}{2.43} & 3.31 & \multicolumn{1}{c|}{65.00} & 95 \\
\cellcolor[HTML]{ECF4FF}N3(5) & \multicolumn{1}{c|}{0.59} & \multicolumn{1}{c|}{3.83} & \multicolumn{1}{c|}{5.44} & 4.22 & \multicolumn{1}{c|}{1.72} & \multicolumn{1}{c|}{4.65} & 3.39 & \multicolumn{1}{c|}{2.02} & 1.56 & \multicolumn{1}{c|}{85.00} & 85 \\
\cellcolor[HTML]{ECF4FF}N5(5) & \multicolumn{1}{c|}{3.32} & \multicolumn{1}{c|}{31.48} & \multicolumn{1}{c|}{23.42} & 26.19 & \multicolumn{1}{c|}{20.02} & \multicolumn{1}{c|}{18.62} & 21.14 & \multicolumn{1}{c|}{0.89} & 0.88 & \multicolumn{1}{c|}{100.00} & 100 \\ \hline
\cellcolor[HTML]{ECF4FF}BT9(2) & \multicolumn{1}{c|}{0.13} & \multicolumn{1}{c|}{0.13} & \multicolumn{1}{c|}{0.77} & 0.42 & \multicolumn{1}{c|}{0.13} & \multicolumn{1}{c|}{0.77} & 0.42 & \multicolumn{1}{c|}{1.10} & 1.41 & \multicolumn{1}{c|}{95.00} & 40 \\
\cellcolor[HTML]{ECF4FF}BT10(2) & \multicolumn{1}{c|}{0.30} & \multicolumn{1}{c|}{0.31} & \multicolumn{1}{c|}{1.16} & 0.46 & \multicolumn{1}{c|}{0.31} & \multicolumn{1}{c|}{1.16} & 0.46 & \multicolumn{1}{c|}{1.13} & 1.27 & \multicolumn{1}{c|}{95.00} & 75 \\
\cellcolor[HTML]{ECF4FF}BT1B(3) & \multicolumn{1}{c|}{34.83} & \multicolumn{1}{c|}{47.58} & \multicolumn{1}{c|}{118.27} & 62.88 & \multicolumn{1}{c|}{47.27} & \multicolumn{1}{c|}{118.11} & 62.70 & \multicolumn{1}{c|}{0.93} & 0.99 & \multicolumn{1}{c|}{100.00} & 100 \\
\cellcolor[HTML]{ECF4FF}BT2(3) & \multicolumn{1}{c|}{5.62} & \multicolumn{1}{c|}{7.08} & \multicolumn{1}{c|}{22.67} & 28.17 & \multicolumn{1}{c|}{6.97} & \multicolumn{1}{c|}{22.61} & 28.11 & \multicolumn{1}{c|}{0.93} & 1.12 & \multicolumn{1}{c|}{100.00} & 95 \\
\cellcolor[HTML]{ECF4FF}BT2B(3) & \multicolumn{1}{c|}{9.58} & \multicolumn{1}{c|}{14.40} & \multicolumn{1}{c|}{41.67} & 21.23 & \multicolumn{1}{c|}{14.13} & \multicolumn{1}{c|}{4.36} & 21.09 & \multicolumn{1}{c|}{1.00} & 1.05 & \multicolumn{1}{c|}{100.00} & 95 \\
\cellcolor[HTML]{ECF4FF}BT3(3) & \multicolumn{1}{c|}{5.38} & \multicolumn{1}{c|}{14.15} & \multicolumn{1}{c|}{62.22} & 32.10 & \multicolumn{1}{c|}{12.80} & \multicolumn{1}{c|}{61.54} & 31.39 & \multicolumn{1}{c|}{1.05} & 1.11 & \multicolumn{1}{c|}{100.00} & 100 \\
\cellcolor[HTML]{ECF4FF}BT5(3) & \multicolumn{1}{c|}{0.06} & \multicolumn{1}{c|}{0.38} & \multicolumn{1}{c|}{0.27} & 0.19 & \multicolumn{1}{c|}{0.12} & \multicolumn{1}{c|}{0.14} & 0.05 & \multicolumn{1}{c|}{0.57} & 0.57 & \multicolumn{1}{c|}{100.00} & 85 \\
\cellcolor[HTML]{ECF4FF}BT8(3) & \multicolumn{1}{c|}{12.17} & \multicolumn{1}{c|}{19.31} & \multicolumn{1}{c|}{169.32} & 79.52 & \multicolumn{1}{c|}{18.87} & \multicolumn{1}{c|}{169.12} & 79.31 & \multicolumn{1}{c|}{1.00} & 1.05 & \multicolumn{1}{c|}{100.00} & 100 \\
\cellcolor[HTML]{ECF4FF}BT8B(3) & \multicolumn{1}{c|}{3.17} & \multicolumn{1}{c|}{3.60} & \multicolumn{1}{c|}{41.63} & 24.23 & \multicolumn{1}{c|}{3.55} & \multicolumn{1}{c|}{41.60} & 24.20 & \multicolumn{1}{c|}{1.16} & 1.20 & \multicolumn{1}{c|}{100.00} & 100 \\
\cellcolor[HTML]{ECF4FF}BT11(3) & \multicolumn{1}{c|}{17.21} & \multicolumn{1}{c|}{35.21} & \multicolumn{1}{c|}{51.19} & 31.23 & \multicolumn{1}{c|}{34.34} & \multicolumn{1}{c|}{50.77} & 30.80 & \multicolumn{1}{c|}{0.88} & 0.95 & \multicolumn{1}{c|}{100.00} & 100 \\
\cellcolor[HTML]{ECF4FF}BT1(4) & \multicolumn{1}{c|}{3.51} & \multicolumn{1}{c|}{97.33} & \multicolumn{1}{c|}{63.69} & 68.39 & \multicolumn{1}{c|}{13.88} & \multicolumn{1}{c|}{18.83} & 22.18 & \multicolumn{1}{c|}{1.02} & 1.04 & \multicolumn{1}{c|}{100.00} & 100 \\
\cellcolor[HTML]{ECF4FF}BT6(4) & \multicolumn{1}{c|}{12.36} & \multicolumn{1}{c|}{124.16} & \multicolumn{1}{c|}{106.73} & 95.75 & \multicolumn{1}{c|}{43.87} & \multicolumn{1}{c|}{61.96} & 50.18 & \multicolumn{1}{c|}{1.01} & 1.03 & \multicolumn{1}{c|}{100.00} & 100 \\
\cellcolor[HTML]{ECF4FF}BT7(4) & \multicolumn{1}{c|}{4.36} & \multicolumn{1}{c|}{95.89} & \multicolumn{1}{c|}{60.65} & 56.06 & \multicolumn{1}{c|}{15.06} & \multicolumn{1}{c|}{15.24} & 9.29 & \multicolumn{1}{c|}{1.00} & 1.01 & \multicolumn{1}{c|}{100.00} & 100 \\ \hline
\cellcolor[HTML]{ECF4FF}UM4(2) & \multicolumn{1}{c|}{2.43} & \multicolumn{1}{c|}{2.93} & \multicolumn{1}{c|}{12.71} & 2.26 & \multicolumn{1}{c|}{2.90} & \multicolumn{1}{c|}{12.69} & 2.24 & \multicolumn{1}{c|}{0.96} & 1.41 & \multicolumn{1}{c|}{70.00} & 5 \\
\cellcolor[HTML]{ECF4FF}UM4B1(2) & \multicolumn{1}{c|}{4.81} & \multicolumn{1}{c|}{5.68} & \multicolumn{1}{c|}{17.38} & 5.03 & \multicolumn{1}{c|}{5.64} & \multicolumn{1}{c|}{17.35} & 5.01 & \multicolumn{1}{c|}{0.86} & 1.12 & \multicolumn{1}{c|}{90.00} & 45 \\
\cellcolor[HTML]{ECF4FF}UM1(3) & \multicolumn{1}{c|}{6.68} & \multicolumn{1}{c|}{58.62} & \multicolumn{1}{c|}{49.35} & 31.92 & \multicolumn{1}{c|}{47.14} & \multicolumn{1}{c|}{42.58} & 24.86 & \multicolumn{1}{c|}{0.98} & 1.09 & \multicolumn{1}{c|}{100.00} & 100 \\
\cellcolor[HTML]{ECF4FF}UM2(3) & \multicolumn{1}{c|}{0.04} & \multicolumn{1}{c|}{2.94} & \multicolumn{1}{c|}{4.49} & 3.44 & \multicolumn{1}{c|}{0.21} & \multicolumn{1}{c|}{2.97} & 1.83 & \multicolumn{1}{c|}{1.95} & 2.23 & \multicolumn{1}{c|}{75.00} & 40 \\
\cellcolor[HTML]{ECF4FF}UM5(3) & \multicolumn{1}{c|}{2.87} & \multicolumn{1}{c|}{31.79} & \multicolumn{1}{c|}{29.71} & 21.31 & \multicolumn{1}{c|}{19.01} & \multicolumn{1}{c|}{22.18} & 13.63 & \multicolumn{1}{c|}{1.05} & 1.26 & \multicolumn{1}{c|}{95.00} & 85 \\ \hline
\cellcolor[HTML]{ECF4FF}ZZB1(2) & \multicolumn{1}{c|}{0.44} & \multicolumn{1}{c|}{0.49} & \multicolumn{1}{c|}{10.44} & 0.47 & \multicolumn{1}{c|}{0.48} & \multicolumn{1}{c|}{10.43} & 0.45 & \multicolumn{1}{c|}{0.89} & 1.01 & \multicolumn{1}{c|}{100.00} & 5 \\
\cellcolor[HTML]{ECF4FF}ZZB2(2) & \multicolumn{1}{c|}{0.71} & \multicolumn{1}{c|}{2.51} & \multicolumn{1}{c|}{272.70} & 2.66 & \multicolumn{1}{c|}{1.22} & \multicolumn{1}{c|}{271.96} & 1.83 & \multicolumn{1}{c|}{0.89} & 2.62 & \multicolumn{1}{c|}{100.00} & 65 \\
\cellcolor[HTML]{ECF4FF}ZZB3(2) & \multicolumn{1}{c|}{0.47} & \multicolumn{1}{c|}{7.69} & \multicolumn{1}{c|}{341.84} & 5.33 & \multicolumn{1}{c|}{1.18} & \multicolumn{1}{c|}{338.15} & 1.45 & \multicolumn{1}{c|}{0.88} & 2.58 & \multicolumn{1}{c|}{100.00} & 65 \\ \hline
\end{tabular}
\caption{Comparison of running time and solution length between $\XA$ (using \loc and \glo) and uniform random sampling. For random sampling we report the average values in terms of running and solution length (the latter is given as normalized value with  respect to the solution length with $\XA$). }
\label{tbl:lattice_vs_random:app}
\end{table*}

\subsection{Comparison with Random Sampling}
Extended results where $\XA$-samples are compared with \rnd are given in Table~\ref{tbl:lattice_vs_random:app}. Here, we consider two versions of random sampling. The first version, denoted by \rnd, which is identical to the one considered in the main paper, uses random sampling together with the asymptotically optimal connection radius $r_{\textup{rnd}}(n)$, which is commonly used in practice. The second version, denoted by \rndm uses the radius as ${r^*}$ used for lattice-based sampling. The latter is used to further emphasize the inferiority of uniform random sampling as compared to $\XA$ due to identical parameters between $\XA$-\glo and \rndm-\glo (except for the sampling distribution). In particular, the move to \rndm  severely reduces the success rates in some of the scenarios.

Another addition in Table~\ref{tbl:lattice_vs_random:app} is the running time of the search algorithm (under "search time"). Recall that the total running time for \glo consists of the (i) construction of the sample set and the nearest-neighbor data structure and the (ii) running the search algorithm. Although both $\XA$-\glo and \rnd use the same number of samples, the construction time is usually larger in the former due to an additional step of constructing the lattice samples over the entire configuration space, which is currently implemented in a naive and unoptimized manner. In this sense, the comparison between $\XA$-\glo and \rnd is not entirely fair. Thus, we also report the running time of the search algorithm, which can be the computational bottleneck, especially for more complicated robot geometries where the collision-check operation is more expensive~\cite{KleinbortSH16}. Although the search time for $\XA$-\glo is usually lower for most scenarios, we argue that with more expensive collision checks, the advantage of lattice-based sample sets would be even more prominent.

\subsection{Effect of parameter choice}
We report the effect of the choice of the $\delta$ and $\eps$ parameters on solution length and running time for the \loc algorithm using $\XA$ sampling. We specifically focus on the ZZB3 scenario due to the availability of several homotopy classes for the solution, where each class has a different length and level of difficulty. For instance, in one class, the robots use the rightmost part of the workspace, which consists of a winding path, and exchange positions halfway between---leading to a relatively short solution length. In a second class, the robots use the long passage to the left, which consists of long straight-line motions and yields a significantly longer solution length.

We set $\delta\in \{2.75,4\}$ and report the solution length and running time in Figure~\ref{fig:parameters:app} for $\eps\in \{0.5,0.6,\ldots,1,2,\ldots,10\}$. Observe that for $\delta=2.75$ the planner obtains a low-length solution already for high $\eps$ values, whereas $\delta=4$ initially uncovers an inefficient solution length-wise but eventually settles on the better homotopy class when $\eps$ is reduced. From values of $\eps\leq 1$ the length relatively stabilizes, while the runtime jumps at several orders of magnitude, which highlights the exponential dependence of sample and collision-check complexity on the value $\eps$. Finding the middle-ground $\eps$ value is an important goal, which we leave for future work. 

Notice that the planner fails to find a solution for $\delta=4$ and $\eps\in\{2,4,5\}$. Due to our \decomps result, this implies that no $4$-clear solution exists. Despite this, the planner does succeed for some values of $\eps$, which suggests that our sufficient conditions for \decomps are not necessary. The success could also be explained by the specific arrangement of the points in $\XA$, which coincidentally induces a connected component via the second homotopy class for this specific scenario. It should also be noted that the sample set $\X_{\AN}^{4,\eps}$ can be viewed (via Lemma~1) as the sample set $\X_{\AN}^{2.5,\eps'}$ for $\eps$ small enough, which explains the success of the planner with  $\delta=4$ and smaller $\eps$ values. 


\fi

\bibliographystyle{plainnat}
\bibliography{references}

\end{document}